\documentclass{article} %
\makeatletter
\let\ltx@orig@addcontentsline\addcontentsline
\makeatother
\usepackage{iclr2027_conference,times}
\makeatletter
\let\addcontentsline\ltx@orig@addcontentsline
\makeatother

\newif\ifpreprint %
\preprinttrue
\ifpreprint\iclrfinalcopy\fi

\usepackage[utf8]{inputenc} %
\usepackage[T1]{fontenc}    %
\usepackage{hyperref}       %
\usepackage{url}            %
\usepackage{booktabs}       %
\usepackage{amsfonts}       %
\usepackage{nicefrac}       %
\usepackage{microtype}      %
\usepackage{xcolor}         %

\usepackage{multirow}
\usepackage{graphicx}
\usepackage{algorithm}
\usepackage{algpseudocode}
\usepackage{amsmath, amssymb}
\usepackage{amsthm}
\usepackage{cleveref}
\crefname{equation}{Eq.}{Eqs.}
\Crefname{equation}{Eq.}{Eqs.}
\usepackage{math}
\usepackage{bm}
\usepackage{mathtools}
\usepackage{tikz}
\usepackage{placeins}

\usepackage{enumitem}

\newcommand{\tactis}{TACTiS-2$^\dagger$}
\usetikzlibrary{positioning, fit, backgrounds, arrows.meta, shadows, calc}
\newtheorem{lemma}{Lemma}
\newtheorem{proposition}[lemma]{Proposition}

\newtheorem{theorem}[lemma]{Theorem}
\newtheorem*{theoremrestate}{Theorem}   %

\title{Gaussian Mixture Copula Processes \\ for Irregular Multivariate Time Series}

\newcommand{\aff}[1]{\textsuperscript{\normalfont#1}}
\author{%
  Christian Klötergens\aff{1,2},\ \ Tom Hanika\aff{1},\\
  \bfseries Lars Schmidt-Thieme\aff{1,2},\ \ Vijaya Krishna Yalavarthi\aff{1,2}\\
  {\normalfont \aff{1}Information Systems and Machine Learning Lab (ISMLL), University of Hildesheim, Germany}\\
  {\normalfont \aff{2}VWFS -- Data Analysis \& Research Center (DARC), University of Hildesheim, Germany}\\
  {\normalfont\ttfamily \{kloetergens, schmidt-thieme, yalavarthi\}@ismll.de}%
}

\begin{document}

\maketitle
\ifpreprint\lhead{Preprint. Under review.}\fi

\begin{abstract}
We introduce \emph{Gaussian Mixture Copula Processes} (\model{}), a conditional copula process for
irregularly sampled multivariate time series (IMTS) that is expressive and marginalization consistent
at the same time. Existing conditional copula processes achieve only one of the two. Gaussian copula
processes are consistent by construction but confined to elliptical dependence, whereas
attentional copulas such as TACTiS-2 are far more expressive but neither preserve the marginals they
are built on nor guarantee that integrating out a target point returns the joint over the remaining
ones. 
\model{} closes this gap by inferring the parameters of a Gaussian Mixture Copula from the context in a way that guarantees consistent marginalization. 
This yields a non-elliptical copula process with tractable likelihood and consistency guarantees. 
Our experiments show that copulas equipped with \margmodel{}, a normalizing-flow model for the
univariate marginals fitted in isolation, outperform jointly trained baselines in probabilistic IMTS
forecasting. We
further provide evidence that \model{} is more expressive than Gaussian copula processes, and on par
with the inconsistent TACTiS-2.
\end{abstract}

\addtocontents{toc}{\protect\setcounter{tocdepth}{-1}}

\section{Introduction}

Probabilistic forecasting of irregularly sampled multivariate time series (IMTS) implies modeling a
stochastic process. Given a context of past observations, a model must assign a predictive distribution to
any finite set of target points, where a target is a (channel, time) pair. Predicting each target point in
isolation is not enough. Many downstream tasks, such as sampling coherent trajectories, necessitate
the \emph{joint} predictive distribution over all target points.
\begin{figure}[b]
    \centering
    \includegraphics[width=\textwidth]{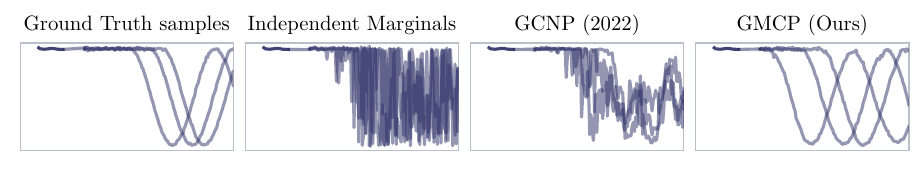}
    \setlength{\abovecaptionskip}{-7pt}
    \caption{Three samples for the trajectory of a pendulum. Independently sampled marginals and a Gaussian Copula Neural
    Process (GCNP)~\citep{Markou2021.Practical} generate incoherent trajectories. In contrast, \model{} samples coherent
    swings.}\label{fig:pendulum}
\end{figure}
Because the target set is chosen at prediction time, a model does not merely output one joint
distribution but a whole family of them, one per target set. 
For this family to describe a valid stochastic process, the Kolmogorov extension
theorem~\citep{Oksendal2003.Stochastic} requires it to be \emph{marginalization consistent}. That is, integrating
a target point out of a larger joint must return exactly the joint over the remaining points. For example, consider a
weather model asked for tomorrow's temperature and snowfall. Marginalizing snowfall out of the joint
forecast must reproduce the temperature forecast the model gives when asked about temperature alone. If not, the answer to ``How cold will it be?'' depends on whether snow was asked about. Hence, the model's predictions cannot all come from one underlying process. It would be hard to trust such a model. 

\emph{Copula processes} are an attractive way to build marginalization consistent models. By Sklar's
theorem~\citep{Sklar1959.Fonctions,Nelsen2006.Introduction}, a joint distribution factorizes into its
univariate marginals and a copula that carries the dependency structure alone. Modeling the two parts
separately buys modeling freedom, which can lead to better predictive performance~\citep{Ashok2023.TACTiS2}.
Additionally, univariate forecasts are explicit and unchanged by the joint model, which matters wherever univariate forecasts are calibrated or certified on their own, as in regulated financial and clinical
settings.
The existing conditional copula processes, however, force a choice between expressivity and
consistency. Gaussian copula processes~\citep{Wilson2010.Copula,Salinas2019.Highdimensional,Markou2021.Practical} are
marginalization consistent by construction, but their dependence structure is that of a multivariate
Gaussian and hence elliptical and unimodal. TACTiS/TACTiS-2~\citep{Drouin2022.TACTiS,Ashok2023.TACTiS2} are far
more expressive, but the attention-based copula is not guaranteed to be marginalization consistent and does not preserve the marginals it is built on.
A copula family that offers more expressivity is the \emph{Gaussian Mixture Copula}
(GM-C)~\citep{Tewari2011.Parametric}. A GM-C is not a mixture of Gaussian copulas, but
the copula induced by a latent Gaussian \emph{mixture}. It therefore can represent dependency structures that lie outside the reach of Gaussian copulas, while retaining a tractable likelihood. So far,
however, GM-Cs have only been studied as \emph{static} models: one copula is fit to one fixed set of
variables from repeated draws of those variables~\citep{Rajan2016.Dependency,Bilgrau2016.GMCM,Tewari2023.Estimation}.
Whether a GM-C can be turned into a stochastic process, conditioned on a context and evaluated on
arbitrary target sets, is unexplored.

We close this gap with \emph{Gaussian Mixture Copula Processes} (\model{}). Rather than fitting one copula to one dataset, an encoder maps each
instance's context and target set to its own copula, so that a different context induces a different
dependency structure, not merely different marginals. Marginalization consistency has to hold while the copula itself varies with
the context and with which points are requested. This restricts how the latent mixture may be
parameterized. The encoder has to be trained by SGD backpropagating through a copula density that requires the
inverse marginal CDF of a Gaussian mixture, for which automatic differentiation is expensive and numerically unstable.
What we obtain is a copula process that is expressive and
marginalization consistent at the same time, rather than one at the cost of the other.
We demonstrate \model{}'s expressivity advantage over Gaussian Copula Neural Processes~\citep{Markou2021.Practical} in \Cref{fig:pendulum}. The latter is not able to model the ground truth distribution due to the non-Gaussian dependence structure, while \model{} resembles the ground truth stochastic process. 
In the context of IMTS forecasting, \model{} learns to couple well-fitted marginals to attain a better joint likelihood than every jointly trained model.
Our contributions can be summarized as follows:

\begin{itemize}[leftmargin=1.5em,itemsep=2pt,topsep=2pt]
    \item With \model{}, we introduce a \textbf{conditional neural copula process} whose dependency structure is
    a Gaussian Mixture Copula. To our knowledge it is the first copula process that is
    simultaneously conditioned on a context set, marginalization consistent, permutation equivariant,
    and able to express non-Gaussian dependence with a tractable likelihood.
    \item We derive a method to train \model{} as a marginalization consistent copula process for IMTS, built on a numeric inverse-CDF solver and using~\citet{Tewari2023.Estimation}'s analytic gradients for GM-Cs.
    \item Using a synthetic dataset and four IMTS benchmarks we provide evidence that \textbf{\model{} is more expressive than Gaussian copula processes}.  
    \item Furthermore, our experiments show that \textbf{\model{} matches TACTiS-2's performance} in
    probabilistic IMTS forecasting. In contrast to TACTiS-2, \model{} is constructed to guarantee marginalization
    consistency. We confirm this empirically by showing that TACTiS-2 changes the univariate marginal distributions, while \model{} preserves them. 
    \item We show that training a model to explicitly predict univariate marginals yields superior marginal likelihood estimation compared to models that are optimized on the joint likelihood.
\end{itemize}

\section{Preliminaries}\label{sec:gmc}
Throughout, uppercase letters denote cumulative distribution functions (CDFs) and the corresponding lowercase letters denote their densities.
A copula is a CDF with domain ${[0,1]}^{N}$ and uniform marginals.
\emph{Sklar's theorem}~\citep{Sklar1959.Fonctions,Nelsen2006.Introduction} states that every joint CDF $F$ with marginals $F_n$ can be written as $F(\mathbf{y}) = \copC(\mathbf{u})$, where $\mathbf{u} = \big(F_1(y_1), \dots, F_N(y_N)\big)$ collects the probability integral transforms of the coordinates, and that the copula $\copC$ is unique whenever the marginals are continuous.
For any continuous density function $f$, its copula density can be written explicitly as:
\begin{equation}\label{eq:copula}
	\cop_f(\mathbf{u}) \coloneqq
	\frac{
		f\big(F_1^{-1}(u_1), \dots, F_N^{-1}(u_N)\big)
	}{
		\prod_{n=1}^N f_n\!\big(F_n^{-1}(u_n)\big)
	}, \qquad \mathbf{u}\in{[0,1]}^{N}.
\end{equation}
Conversely, one can define a joint distribution by combining separate univariate density functions $g_n(y_n)$ with a copula density function:
$p(\mathbf{y}) = \cop\big(G_1(y_1),\dots,G_N(y_N)\big)\prod_{n=1}^N g_n(y_n)$.

A well-established family of copula density functions is the \emph{Gaussian copula}, obtained by applying~\Cref{eq:copula} to a multivariate normal density $\mathcal{N}(\mu,\Sigma)$. As it depends neither on the mean $\mu$ nor on the scaling of the marginal variances, it is sufficient to consider multivariate normal distributions with mean $\mu=0$ and a covariance matrix $\Sigma$ with all $1$'s on its diagonal, called a correlation matrix; thus the Gaussian copula is denoted simply by $\cop_{\Sigma} \coloneqq \cop_{\mathcal{N}(0,\Sigma)}$. It can only represent the dependence structure of a multivariate Gaussian.
As mixtures of copula densities are also copula densities, a more expressive class of copula densities can be constructed as \emph{mixtures of Gaussian copulas} (Mix-GC)~\citep{Marbac2017.Modelbased}:
\begin{equation}\label{eq:mixgauss}
	\cop_{\pi,\Sigma_1,\dots,\Sigma_K}(\mathbf{u}) \coloneqq \sum_{j=1}^{K} \pi_j\, \cop_{\Sigma_j}(\mathbf{u}),
	\qquad \pi_j \ge 0, \quad \sum_{j=1}^{K} \pi_j = 1.
\end{equation}
This construction is, however, fundamentally restricted.
\citet{Khaled2023.Approximating} showed that mixtures of Gaussian copulas are
always radially symmetric on the unit cube, i.e.\ invariant under
$\mathbf{u} \mapsto \mathbf{1} - \mathbf{u}$. A second obstruction concerns the
center $\mathbf{u}^\star = (\nicefrac{1}{2},\dots,\nicefrac{1}{2})$ of the
cube. Each component of~\Cref{eq:mixgauss} is standardized by its own
marginals and, being a Gaussian copula, discards the mean, so $\mathbf{u}^\star$
is the componentwise median of every component alike and none of them can be
placed off-center. Component $j$ therefore contributes
$\lvert\Sigma_j\rvert^{-1/2} \ge 1$ at $\mathbf{u}^\star$ whatever its
parameters, and hence $\cop(\mathbf{u}^\star) \ge 1$ however many components
are used. Hence, a mixture of Gaussian copulas can raise the density at the center
above the independence value $1$, but never lower it
(\Cref{app:gmc-expressivity}). Copulas with a hole at the center, such as
that of the circle in \Cref{fig:xo}, are therefore out of reach.

\textbf{Gaussian Mixture Copulas.}
A \emph{Gaussian Mixture Copula} (GM-C)~\citep{smillie2008new, Tewari2011.Parametric} is \textbf{not} a mixture of Gaussian copulas.
A GM-C first defines a Gaussian Mixture Model (GMM):
\begin{equation}
	\pdfgmm(z) \;=\; \sum_{j=1}^{K} \pi_j\, \mathcal{N}(z;\mu_j,\Sigma_j),
	\qquad \pi_j \ge 0, \quad \sum_{j=1}^{K} \pi_j = 1,
\end{equation}
and maps $u_n \in [0,1]$ to the latent space via the marginal inverse CDFs:
$z_n = \cdfgmmninv{n}(u_n)$, where $\cdfgmmn{n}$ denotes the marginal CDF of the Gaussian mixture $\pdfgmm$.
The copula is then obtained by applying~\Cref{eq:copula} to $\pdfgmm$, whereas~\Cref{eq:mixgauss} applies it to each component $g_j$ separately. The difference is the order of operations:
$
\sum_{j=1}^{K} \pi_j\, \cop_{g_j}
\;\neq\;
\cop_{\sum_{j=1}^{K} \pi_j g_j}. %
$
Mixing after the transformation forces every component through the same standard normal marginals, whereas mixing before it lets the marginal transformations $\cdfgmmninv{n}$ depend on the full mixture, so that the induced copula inherits the multimodal structure of $\pdfgmm$. This enables GM-Cs to represent complex, non-elliptical dependence patterns and the circle distribution in \Cref{fig:xo}.

\begin{figure}[t]
    \centering
    \includegraphics[width=\textwidth,trim={0 6pt 0 0},clip]{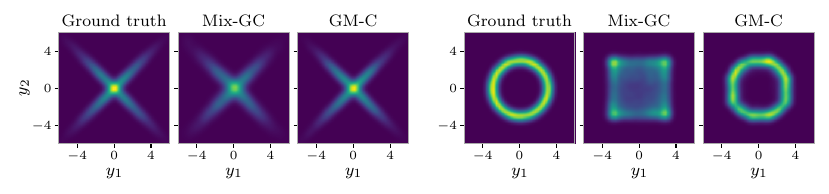}
    \setlength{\abovecaptionskip}{-7pt}
    \caption{Comparing the expressiveness of a Mixture of Gaussian Copulas (Mix-GC) with a Gaussian Mixture Copula (GM-C) on two 2-D toy datasets.
    (More details are in \Cref{app:gmc-expressivity}).}\label{fig:xo}
\end{figure}

\section{Gaussian Mixture Copula Processes (\model{})}\label{sec:model}

We now introduce \model{}, a conditional copula process whose dependency structure is a Gaussian Mixture Copula.
We define it in the general setting of neural processes.

\textbf{Neural Processes and Copula Processes.} A conditional neural process~\citep{Garnelo2018.Conditional} predicts outputs at freely chosen inputs, given a set of already observed input-output pairs. %
Let $\mathcal{X}$ denote the input space. We restrict outputs to be scalar and real-valued: $\ytgt \in \R$.
A \emph{context set}
$
    \Dctx = {\{(\xctx_m, \yctx_m)\}}_{m=1}^{M}$ with $ \xctx_m \in \mathcal{X},\; \yctx_m \in \R,
$
collects what has been observed, and a \emph{target set} $\Xtgt = (\xtgt_1, \dots, \xtgt_N)$ with $\xtgt_n \in \mathcal{X}$ collects the inputs at which a prediction is requested.
We write $\Ytgt = (\ytgt_1, \dots, \ytgt_N) \in \R^N$ for the corresponding target outputs.
The model has to supply a joint predictive density $\hat p(\Ytgt \mid \Xtgt, \Dctx)$ for every context set and every finite target set.
Because $\Xtgt$ is chosen at prediction time, this is not one joint distribution but a whole family of them, one per target set.
For the family to be the finite-dimensional distributions of a single stochastic process conditioned on $\Dctx$, the Kolmogorov extension theorem~\citep{Oksendal2003.Stochastic} requires it to be permutation equivariant and \emph{marginalization consistent}: for every $k \in \{1, \dots, N\}$,
\begin{equation}
    \int_{\R} \hat p(\Ytgt \mid \Xtgt, \Dctx)\, d\ytgt_k \;=\; \hat p(\Ytgt_{-k} \mid \Xtgt_{-k}, \Dctx), \label{eq:mc-def}
\end{equation}
where $\Xtgt_{-k}$ and $\Ytgt_{-k}$ denote $\Xtgt$ and $\Ytgt$ with the $k$-th entry removed. The property extends to arbitrary subsets by induction.
In their original setup, conditional neural processes~\citep{Garnelo2018.Conditional} satisfy~\Cref{eq:mc-def} in the simplest possible way: they factorize the joint across target points, $\hat p(\Ytgt \mid \Xtgt, \Dctx) = \prod_{n=1}^{N} \hat p(\ytgt_n \mid \xtgt_n, \Dctx)$, and thereby declare the target points conditionally independent, so that no dependence is modeled at all.
The common remedy is an autoregressive factorization, applied either at deployment time~\citep{Bruinsma2022.Autoregressive} or already during training~\citep{Nguyen2022.Transformer}.
While this does induce dependency, it sacrifices permutation equivariance and marginalization consistency.

A \emph{copula process}~\citep{Wilson2010.Copula} is a neural process whose predictive density is built by combining the univariate marginal CDFs of the targets, ${\mathbf{u} = \bigl(u_n = \cdfmargn{n}(\ytgt_n \mid \xtgt_n, \Dctx)\bigr)}_{n=1}^{N}$, with a copula density $\cop$:
\begin{align}
    \hat p(\Ytgt \mid \Xtgt, \Dctx)
    &= \cop(\mathbf{u} \mid \Xtgt, \Dctx)
    \prod_{n=1}^{N} \pdfmargn{n}(\ytgt_n \mid \xtgt_n, \Dctx). \label{eq:copproc}
\end{align}

\textbf{Conditioning the Gaussian Mixture Copula.}
To the best of our knowledge, \model{} is the first copula process that is marginalization consistent and expresses non-Gaussian dependence with a tractable likelihood.
Building a GM-C conditional on the context ($\Dctx$) and target set ($\Xtgt$) necessitates inferring the parameters ${\{\pi_j, \mu_j, \Sigma_j\}}_{j=1}^K$ of this latent Gaussian mixture from $\Dctx$ and $\Xtgt$. 

The mixture weights are inferred from the context alone,
$\pi = \phi_\pi(\Dctx)$.
In particular, $\pi$ is not informed about which target points are requested.
This is intended: a mixture whose weights depend on the context alone stays marginalization consistent (\Cref{lem:moses-32}, \Cref{app:mc-proof}), whereas weights that changed with $\Xtgt$ would no longer commute with integrating out a target point.
Each component $j$ of the Gaussian mixture has a mean vector $\mu_j \in \R^N$ and per-target standard deviations $\sigma_{j,n} > 0$, with one entry per target point. We infer the $n$-th entry of all $K$ components jointly, from the $n$-th target input and the context,
\begin{equation}
    \bigl[\mu_{1,n}, \dots, \mu_{K,n}\bigr] = \phi_\mu(\xtgt_n, \Dctx) \in \R^{K},
    \qquad
    \bigl[\sigma_{1,n}, \dots, \sigma_{K,n}\bigr] = \phi_\sigma(\xtgt_n, \Dctx) \in \R_{>0}^{K}. \label{eq:mu-sigma}
\end{equation}
Letting the components have their own means is what gives \model{} its expressivity beyond mixtures of Gaussian copulas: because $\mu_{j,n} \neq \mu_{j',n}$ in general, the components of the latent mixture can have spatially separated modes, and the induced copula can therefore represent more complex dependencies.
The off-diagonal entries of $\Sigma_j$ encode the dependence between target points. They are inherently pairwise, yet each $\Sigma_j$ has to stay positive definite. We therefore do not infer them directly, but infer a feature vector per target point and component,
\begin{equation}
    \bigl[\mathbf{v}_{1,n}, \dots, \mathbf{v}_{K,n}\bigr] = \phi_V(\xtgt_n, \Dctx) \in \R^{K \times H}, \label{eq:V}
\end{equation}
and read the dependence off their Gram matrix. We first normalize each feature vector,
    $\tilde{\mathbf{v}}_{j,n} = \frac{\mathbf{v}_{j,n}}{\sqrt{1 + {\lVert \mathbf{v}_{j,n} \rVert}^2}}$,
which places it strictly inside the unit ball. Stacking $\tilde V_j = {[\tilde{\mathbf{v}}_{j,1}, \dots, \tilde{\mathbf{v}}_{j,N}]}^\top \in \R^{N \times H}$, we complete the diagonal of its Gram matrix to ones, which makes it a correlation matrix, and rescale by the inferred standard deviations,
\begin{equation}
    R_j = \tilde V_j \tilde V_j^\top + \mathrm{diag}{\bigl(1 - {\lVert \tilde{\mathbf{v}}_{j,n} \rVert}^2\bigr)}_{n=1}^{N},
    \qquad
    \Sigma_j = D_j R_j D_j, \label{eq:sigma-j}
\end{equation}
with $D_j = \mathrm{diag}(\sigma_{j,1}, \dots, \sigma_{j,N})$.
The normalization keeps every diagonal entry of the completion strictly positive, so $R_j$ is a positive definite correlation matrix, and $\Sigma_j$ is symmetric positive definite and carries the inferred variances $\sigma_{j,n}^2$ on its diagonal.

\begin{theorem}\label{thm:mc}
\model{}, paired with any model that models univariate marginals independently, is marginalization consistent. Hence, for every context set $\Dctx$ and every target set $\Xtgt$, the predicted joint density satisfies marginalization consistency as defined in ~\Cref{eq:mc-def}.
\end{theorem}
\emph{Proof sketch.} The mixture weights $\pi$ are inferred from $\Dctx$ alone,  the $n$-th entries of $\mu_j$ and $\sigma_j$ depend only on $\xtgt_n$ and $\Dctx$, and each entry of $\Sigma_j \in \R^{N \times N}$  depends only on the two target points it indexes.
Removing a target point from $\Xtgt$ therefore deletes exactly the corresponding entry of $\mu_j$ and the corresponding row and column of $\Sigma_j$, and marginalizing it out reduces to the standard rule for a multivariate Gaussian.~\Cref{app:mc-proof} proves it formally.

\textbf{Training.}
Following previous work~\citep{Yalavarthi2025.Probabilistic,Yalavarthi2025.Reliable}, we fit \model{} by minimizing the normalized joint negative log-likelihood (njNLL). Since the predicted joint density factorizes as in~\Cref{eq:copproc}, its logarithm splits into a copula term and a sum of marginal log-densities, and the copula term in turn unfolds via~\Cref{eq:copula}:
\begin{align}
    \mathcal{L}_{\mathrm{njNLL}}
    &= -\frac{1}{N}\log \hat p(\Ytgt\mid \Xtgt, \Dctx), \label{eq:njnll}\\
    &= -\frac{1}{N} \Bigl( \log \cop(\mathbf{u} \mid \Xtgt, \Dctx)
    + \sum_{n=1}^{N} \log \pdfmargn{n}(\ytgt_n \mid \xtgt_n, \Dctx)\Bigr), \label{eq:logsplit}\\
    \log
    \cop(\mathbf{u} \mid \Xtgt, \Dctx)
    &= \log \pdfgmm(\mathbf{z}) - \sum_{n=1}^{N} \log \pdfgmmn{n}(z_n). \label{eq:coplog}
\end{align}
Here, $\pdfgmm(\mathbf{z}) = \sum_{j=1}^{K} \pi_j\,\mathcal{N}(\mathbf{z};\mu_j,\Sigma_j)$ is the latent Gaussian mixture density evaluated at $z_n = \cdfgmmninv{n}(u_n)$, and $\pdfgmmn{n}$ is its $n$-th one-dimensional marginal.
\model{}'s parameters affect only the copula term, and they see the marginal model only through the values $u_n$ it supplies.
Fitting the copula is therefore exactly minimizing the normalized copula log-density on whichever $\mathbf{u} \in {(0,1)}^N$ is handed to it, and no gradient has to be propagated through $\cdfmargn{n}$.

Evaluating the copula log-density requires $z_n = \cdfgmmninv{n}(u_n)$, but the GMM marginal CDF cannot be inverted analytically. We obtain $z_n$ by a safeguarded Newton-bisection scheme in the forward pass (\Cref{app:numerical-icdf}), and compute gradients of $\cdfgmmninv{n}$ with respect to the GMM parameters analytically~\citep{Tewari2023.Estimation}. \Cref{prop:icdf-grad} states the resulting closed forms, and \Cref{app:icdf-grad} derives them step by step; the backward pass composes them with the Jacobians of $\phi_\pi$, $\phi_\mu$ and $\phi_\sigma$.

\begin{proposition}[Gradients of the GMM inverse CDF]\label{prop:icdf-grad}
Let $z_n = \cdfgmmninv{n}(u_n)$, and let $\Phi$ and $\varphi$ denote the standard-normal CDF and PDF\@. Then the gradients with respect to the GMM parameters are:
\begin{equation*}
    \frac{\partial z_n}{\partial \pi_j}
    = -\frac{\Phi\!\left(\frac{z_n-\mu_{j,n}}{\sigma_{j,n}}\right)}{\pdfgmmn{n}(z_n)},
    \quad
    \frac{\partial z_n}{\partial \mu_{j,n}}
    = \frac{\pi_j\,\varphi\!\left(\frac{z_n-\mu_{j,n}}{\sigma_{j,n}}\right)}{\sigma_{j,n}\,\pdfgmmn{n}(z_n)},
    \quad
    \frac{\partial z_n}{\partial \sigma_{j,n}^2}
    = \frac{\pi_j\,\varphi\!\left(\frac{z_n-\mu_{j,n}}{\sigma_{j,n}}\right)\,(z_n-\mu_{j,n})}{2\,\sigma_{j,n}^3\,\pdfgmmn{n}(z_n)}
\end{equation*}
\end{proposition}
These gradients agree with finite differences of the solver output, and the remaining inversion error changes the njNLL by less than $5\times10^{-5}$ on average (\Cref{app:icdf-validation}). The numerical CDF inversion has little impact on \model{}'s computational complexity, which is analyzed in \Cref{app:complexity}.

\textbf{Sampling.}
Given $\Dctx$ and $\Xtgt$, we first apply $\phi_\pi, \phi_\mu, \phi_\sigma, \phi_V$ to parameterize the GM-C\@.
We then draw a sample $\hat{\mathbf{z}} \in \R^N$ from this mixture by first sampling a component index $j \sim \mathrm{Categorical}(\pi)$ and then $\hat{\mathbf{z}} \sim \mathcal{N}(\mu_j, \Sigma_j)$.
Next, we transform $\hat{\mathbf{z}}$ to the unit cube using the GMM's own marginal CDFs $\cdfgmmn{n}$, that is $\hat{u}_n = \cdfgmmn{n}(\hat{z}_n)$, where each $\cdfgmmn{n}$ is itself a one-dimensional Gaussian mixture CDF and is available in closed form.
Finally, we invert the marginal model on each coordinate to obtain $\hat{y}^{\tgt}_n = \cdfmargninv{n}(\hat{u}_n)$.

\textbf{Limitations.}
First, each $\Sigma_j$ has a rank-$H$ off-diagonal Gram structure, so \model{} cannot represent near-full-rank dependence when $H \ll N$. In practice this ceiling is rarely binding (real-world dependencies are rarely full-rank) and improves scaling in $N$ (see \Cref{app:sensitivity,app:complexity}).
Second, to guarantee marginalization consistency, the mixture weights $\pi$ depend solely on the context $\Dctx$, not on the target set $\Xtgt$. In \Cref{app:target-weights} we show that this is only a small restriction in practice. 

\section{Application to Irregular Multivariate Time Series}\label{sec:imts}

In this section we instantiate the parameter functions $\phi_\pi, \phi_\mu, \phi_\sigma, \phi_V$ and the model for the univariate marginals $\cdfmargn{n}$, so that \model{} can be used for IMTS forecasting.

\textbf{IMTS forecasting as a neural process.}
An IMTS instance over $C$ channels is a set of observations ${\{(t_m, c_m, \yctx_m)\}}_{m=1}^{M}$, where $t_m \in \R_{\geq 0}$ is a timestamp, $c_m \in \{1, \dots, C\}$ a channel index, and $\yctx_m \in \R$ the measured value. Channels are observed asynchronously and often only a handful of times over the whole observation window. Forecasting asks for the values at a set of future (time, channel) pairs.
This is precisely the neural process setup with
    $\mathcal{X} = \R_{\geq 0} \times \{1, \dots, C\}$, and 
    $\xtgt_n = (t_n, c_n)$,
so that the context set $\Dctx = {\{((t_m, c_m), \yctx_m)\}}_{m=1}^{M}$ is the observed history and the target set $\Xtgt = (\xtgt_1, \dots, \xtgt_N)$ collects the (time, channel) pairs to be forecast, with outputs $\Ytgt \in \R^N$. Both sets are unstructured. The channel and time points the targets refer to, and number of targets $N$, vary per instance and are revealed only at prediction time.

\textbf{Encoder and parameter heads.}
We adopt the encoder from CircuITS~\citep{Kloetergens2026.Probabilistic}, the state of the art for joint IMTS forecasting. Crucially, this encoder infers the target encodings independent of each other, and is therefore \emph{target-separable}, a crucial property to allow for marginalization consistency. More details are in \Cref{app:encoder}.  
As an intermediate step the encoder produces a per-channel summary of the history $\tilde{\mathbf{H}} \in \R^{C \times D}$, which we aggregate into a global summary $\bar{\mathbf{h}}$ by linear attention with a single attention-query:
\begin{equation}
    \e_n = \enc\bigl(\xtgt_n, \Dctx\bigr) \in \R^{D},
    \qquad
    \bar{\mathbf{h}} = \mathrm{Pool}\bigl(\tilde{\mathbf{H}}\bigr) \in \R^{D}. \label{eq:enc}
\end{equation}
Since $\tilde{\mathbf{H}}$ is a function of the history alone, so is $\bar{\mathbf{h}}$, as required for the mixture weights.
On top of these two representations, the parameter functions of \Cref{sec:model} are 2-layer MLP heads: $\phi_\pi$ reads the global summary $\bar{\mathbf{h}}$ and passes it through a softmax, while $\phi_\mu$, $\phi_\sigma$ and $\phi_V$ read the target encoding $\e_n$, with a softplus on $\phi_\sigma$ that enforces positive standard deviations.
\model{} and the marginal model apply two fully separated instances of the encoder, with potentially different hyperparameters.

\textbf{Marginals.}
Existing marginal models for IMTS~\citep{DeBrouwer2019.GRUODEBayes,Bilos2021.Neural,Schirmer2022.Modeling} rely on Neural ODEs, which have been shown to be ineffective in forecasting~\citep{Klotergens2024.PhysiomeODE}, and restrict each predictive distribution to be Gaussian. 
We therefore take the strongest available joint model CircuITS~\citep{Kloetergens2026.Probabilistic} and reduce it to a univariate model by dropping all components that are designed to model dependency. What is left is the encoding $\e_n$ of~\Cref{eq:enc} with a Deep Sigmoidal Flow (DSF) head~\citep{Huang2018.Neural,Papamakarios2021.Normalizing},
\begin{equation}
    \cdfmargn{n}(\ytgt_n \mid \xtgt_n, \Dctx) = T_{\theta_n}(\ytgt_n),
    \qquad \theta_n = \mathrm{MLP}_{\mathrm{DSF}}(\e_n),
    \label{eq:margflow}
\end{equation}
where $T_{\theta_n} : \R \to (0,1)$ is a strictly monotone univariate normalizing flow. 
The flow makes the marginals non-Gaussian while keeping their CDF in closed form, and conditioning $\theta_n$ on $\e_n$ alone keeps them target separable and hence marginalization consistent.
We refer to the resulting model as  \margmodel{} and refer to \Cref{app:margflow} for more details. 
While \model{} and \margmodel{} have encoders with the same architecture, the encoders have separated weights.
\Cref{fig:arch} in \Cref{app:arch} illustrates how  \margmodel{} works together with \model{}.

\section{Experiments}\label{sec:experiments}
Our experiments are designed to test whether \model{} captures dependence beyond Gaussian copulas and
how two-stage estimation compares to jointly trained models and TACTiS-2~\citep{Ashok2023.TACTiS2}.

\subsection{Modeling Non-Gaussian Processes}\label{sec:synthetic}
\begin{table}[t]
\centering
\caption{\textbf{njNLL} on synthetic processes; lower is better. All
copulas share the same frozen \margmodel{} marginals, so the differences are due to the copula
alone. We present the mean and standard deviation of 5 runs. The best model is marked in
\textbf{bold}.}\label{tab:synthetic}
\small
\setlength{\tabcolsep}{3mm}
\begin{tabular}{lcccc}
\toprule
& Independent Marginals & GCNP & Mix-GC & \textbf{\model{}} (ours) \\
\midrule
Pendulum & \metric{0.918}{0.002} & \metric{0.244}{0.013} & \metric{0.085}{0.006} & \metricb{-0.962}{0.035} \\
L\'evy (CPP) & \metric{0.714}{0.073} & \metric{-0.203}{0.080} & \metric{-0.252}{0.075} & \metricb{-0.411}{0.054} \\
L\'evy (Gamma) & \metric{2.047}{0.003} & \metric{0.895}{0.003} & \metric{0.850}{0.004} & \metricb{0.552}{0.004} \\
\bottomrule
\end{tabular}
\end{table}

\model{} is built to express dependence that no Gaussian copula process can represent. We
demonstrate that our copula model actually achieves this by comparing it with Gaussian copulas and
mixtures of those on the task of modeling non-Gaussian processes. 
In the \emph{pendulum} process of
\Cref{fig:pendulum}, a pendulum is held at a certain height and released from rest at an unobserved
time. Its context covers the stretch in which it is held, hence the model gets to know the initial height.
The fact that the release time is random results in bivariate marginals with non-Gaussian dependence and the relatively poor performance of the 
Gaussian Copula Neural Process (GCNP)~\citep{Markou2021.Practical} or mixtures of those.
Additionally, we evaluate the copula processes on two pure-jump L\'evy
processes~\citep{Sato1999.Levy}: a compound
Poisson process (CPP) with unit jumps and a Gamma process~\citep{Cont2003.Financial}. Because the jumps land at random times after the context
ends, two targets have a joint distribution with many modes spread across the unit
square, whereas every component of a mixture of Gaussian copulas sits at the center of that square
(\Cref{prop:mix-gc-center}).
\Cref{app:pendulum,app:levy} give the full specification of all three processes.
All models share one frozen \margmodel{}, so the columns of \Cref{tab:synthetic} differ in the
copula alone: none at all (independent marginals), GCNP~\citep{Markou2021.Practical} (implemented by \model{} with $K=1$), Mix-GC and \model{},
the latter two with $K = 10$ components. 
The results from \Cref{tab:synthetic} provide evidence that \model{} can learn to build dependence structures beyond the range of Gaussian copulas (or mixtures of those).

\subsection{IMTS Forecasting}\label{sec:forecasting}
\textbf{Datasets.}
Our evaluation uses the four benchmark IMTS datasets: \texttt{USHCN}~\citep{Menne2016.LongTerm}, \texttt{PhysioNet'12}~\citep{Silva2012.Predicting}, \texttt{MIMIC-III}~\citep{Johnson2016.MIMICIII}, and \texttt{MIMIC-IV}~\citep{Johnson2023.MIMICIV}. Crucially, the well-established preprocessing, binning, and split protocol is adopted directly from prior work~\citep{Bilos2021.Neural,Yalavarthi2025.Probabilistic,Yalavarthi2025.Reliable,Kloetergens2026.Probabilistic} to provide a fair comparison.
This protocol trains and evaluates each model five times on different train, validation, and test splits. Importantly, every model is evaluated on the same five splits, so the runs are paired across methods. We refer to~\Cref{app:exp-details} for more details. 
\begin{table}[b]
\centering
\caption{Comparing \textbf{mNLL} (Marginal Negative Log-Likelihood) across datasets. 
Lower values indicate better performance. We present the mean and standard deviation of 5 runs. The best model is marked in \textbf{bold}, the second best is \underline{underlined}. \Cref{tab:mnll-full} adds the baselines omitted here.}\label{tab:mnll}
\small
\setlength{\tabcolsep}{1.5mm}
\begin{tabular}{llrrrr}
\toprule
& \textbf{Model} & \multicolumn{1}{c}{\texttt{USHCN}} & \multicolumn{1}{c}{\texttt{PhysioNet'12}} & \multicolumn{1}{c}{\texttt{MIMIC-III}} & \multicolumn{1}{c}{\texttt{MIMIC-IV}} \\
\midrule
\multirow{4}{*}{\rotatebox{90}{Joint}}
& ProFITi & \metric{-3.324}{0.206} & \metric{-0.016}{0.085} & \metric{0.408}{0.030} & \metric{0.500}{0.322} \\
& MOSES & \metric{-3.355}{0.156} & \metric{-0.271}{0.028} & \metric{0.163}{0.026} & \metric{-0.634}{0.017} \\
& CircuITS & \metricu{-3.717}{0.201} & \metricu{-0.287}{0.037} & \metricu{0.095}{0.107} & \metricu{-0.731}{0.063} \\
& \model{}-Joint & \metric{-3.663}{0.210} & \metric{-0.176}{0.470} & \metric{0.108}{0.238} & \metric{-0.677}{0.044} \\
\midrule
& \textbf{\margmodel{}} & \metricb{-3.948}{0.294} & \metricb{-0.479}{0.021} & \metricb{-0.158}{0.205} & \metricb{-0.833}{0.077} \\
\bottomrule
\end{tabular}
\end{table}

\textbf{Baselines.}
On the joint side we use ProFITi~\citep{Yalavarthi2025.Probabilistic}, MOSES~\citep{Yalavarthi2025.Reliable}, and CircuITS~\citep{Kloetergens2026.Probabilistic}.
As a copula baseline we use the attentional copula of TACTiS-2. Since TACTiS-2's~\citep{Ashok2023.TACTiS2} original implementation was infeasible for sparse IMTS, we adapt it to our setting in the way most favorable to it. We keep its attentional copula and its two-stage training protocol, supply it with the same \margmodel{} marginals \model{} uses, and give it the same IMTS encoder. This adaptation, denoted \tactis{}, isolates the copula module (\Cref{app:tactis-adapt}). Details on hyperparameters are given in \Cref{app:hyperparams}. As a second copula baseline we report Mix-GC, the ablation of \model{} that replaces the Gaussian Mixture Copula by a mixture of Gaussian copulas (\Cref{app:mix-gc}). To measure the effect of training marginals and copula in isolation we also add \model{}-Joint to the baselines. \model{}-Joint does not change the architecture of the compound of \margmodel{} and \model{}, but the two components are optimized together.
We report numbers from previous work~\citep{Yalavarthi2025.Reliable,Kloetergens2026.Probabilistic} whenever appropriate. 

\textbf{Estimating the marginals in isolation pays off.}
\Cref{tab:mnll} reports the marginal Negative Log-Likelihood (mNLL, defined in~\Cref{eq:mnll}, \Cref{app:metrics}); \margmodel{} attains the lowest mNLL on all four datasets.
The comparison that isolates the effect of the two-stage training is the one with \model{}-Joint, the very same compound of \margmodel{} and \model{}, optimized jointly instead of in stages, and worse on every dataset.
The same holds against CircuITS, whose induced marginals are a mixture of flows of exactly the form \margmodel{} uses and are therefore at least as expressive (\Cref{app:margflow}). Against the Neural-ODE based models listed in \Cref{app:full-results}, \Cref{tab:mnll-full} the margin is larger still, as those are restricted to Gaussian predictions.

\textbf{Comparing joint density modeling.}
\Cref{tab:njnll} reports njNLL, the objective every model is trained and tuned on.
\model{} surpasses all jointly trained and univariate models on all four datasets (\Cref{tab:njnll-full}).
Because \Cref{tab:njnll} also reports the njNLL of \margmodel{} with independent marginals, it shows how much of the njNLL stems from the copula alone. All copulas in this experiment share the exact same 
\margmodel{} instances, and are therefore directly comparable in njNLL\@.
Comparing to TACTiS-2, \model{} has a lower average njNLL on three out of four datasets. Furthermore, \Cref{tab:njnll} indicates that \model{}'s expressivity advantage over Mix-GC actually leads to better dependence modeling on real-world and not only synthetic data. The combination of \model{} and \margmodel{} provides a consistently better likelihood estimation, when trained in separate stages, compared to joint training. We provide the mNLL and njNLL of additional but non-competitive baselines in \Cref{app:full-results}.
\begin{table}[t]
\centering
\caption{Comparing \textbf{njNLL} (Normalized Joint Negative Log-Likelihood) across datasets.
Lower values indicate better performance. We present the mean and standard deviation of 5 runs. 
The best model is marked in \textbf{bold}, the second best is \underline{underlined}. \Cref{tab:njnll-full} adds the baselines omitted here.}\label{tab:njnll}
\small
\setlength{\tabcolsep}{1.5mm}
\begin{tabular}{llrrrr}
\toprule
& \textbf{Model} & \multicolumn{1}{c}{\texttt{USHCN}} & \multicolumn{1}{c}{\texttt{PhysioNet'12}} & \multicolumn{1}{c}{\texttt{MIMIC-III}} & \multicolumn{1}{c}{\texttt{MIMIC-IV}} \\
\midrule
& \margmodel{} (independent) & \metric{-3.969}{0.305} & \metric{-0.442}{0.021} & \metric{-0.505}{0.208} & \metric{-1.819}{0.065} \\
\midrule
\multirow{4}{*}{\rotatebox{90}{Joint}}
& ProFITi & \metric{-3.226}{0.225} & \metric{-0.647}{0.078} & \metric{-0.377}{0.032} & \metric{-1.777}{0.066} \\
& MOSES & \metric{-3.357}{0.176} & \metric{-0.491}{0.041} & \metric{-0.305}{0.027} & \metric{-1.668}{0.097} \\
& CircuITS & \metric{-3.789}{0.218} & \metric{-0.550}{0.013} & \metric{-0.574}{0.080} & \metricu{-2.113}{0.044} \\
& \model{}-Joint & \metric{-3.907}{0.269} & \metric{-0.489}{0.399} & \metric{-0.508}{0.309} & \metric{-2.019}{0.049} \\
\midrule
\multirow{3}{*}{\rotatebox{90}{Copula}}
& \tactis{} & \metricb{-4.254}{0.294} & \metricu{-0.732}{0.115} & \metricu{-0.780}{0.282} & \metric{-2.053}{0.170} \\
& Mix-GC & \metric{-4.078}{0.308} & \metric{-0.728}{0.018} & \metric{-0.779}{0.216} & \metric{-2.080}{0.064} \\
& \textbf{\model{}} & \metricu{-4.135}{0.304} & \metricb{-0.745}{0.019} & \metricb{-0.835}{0.218} & \metricb{-2.201}{0.058} \\
\bottomrule
\end{tabular}%
\end{table}

\textbf{Quantifying that \model{} is \emph{more valid} than TACTiS-2.}
To quantify the validity gap, we draw $1000$ samples from \margmodel{} and from the full copula setup, and measure the univariate Wasserstein-1 distance between the two sample sets (see \Cref{app:wd} for the definition). A copula must leave the marginals untouched, so any distance beyond sampling noise is a violation.
The \emph{Control} row is not a separate copula model: it is obtained by sampling twice from the same \margmodel{} instance using different random seeds. Therefore, it represents the irreducible sampling error. 
As shown in \Cref{tab:wd}, \model{} is either within standard deviations of the control or very close to it. 
The WD of TACTiS-2, however, is between five and forty times the control, indicating that it did not learn to preserve the marginal distributions. 
The experiment above probes only the \emph{univariate} marginalization consistency, which is the defining property of a copula. Due to the page limit, we refer to \Cref{app:tactis-breaker} for a more in-depth analysis of TACTiS-2's marginalization inconsistency.
\begin{table}[b]
\centering
\caption{Wasserstein distances (in units of $10^{-3}$) for 1000 samples from \margmodel{} (alone) and the \margmodel{} + copula setup, averaged over 5 splits. The Control row quantifies the sampling error calculated based on sampling twice from \margmodel{} with different random seeds.}\label{tab:wd}
\small
\setlength{\tabcolsep}{3mm}
\begin{tabular}{lcccc}
\toprule
\textbf{Model} & \multicolumn{1}{c}{\texttt{USHCN}} & \multicolumn{1}{c}{\texttt{PhysioNet'12}} & \multicolumn{1}{c}{\texttt{MIMIC-III}} & \multicolumn{1}{c}{\texttt{MIMIC-IV}} \\
\midrule
\tactis{}         & \metric{27.601}{14.465}    & \metric{10.775}{\pz2.978} & \metric{15.724}{\pz8.570} & \metric{\pz2.939}{\pz1.133} \\
\model{} (GM-C)  & \metric{\pz5.083}{\pz2.444} & \metric{\pz0.535}{\pz0.131} & \metric{\pz0.524}{\pz0.079} & \metric{\pz0.381}{\pz0.081} \\
\midrule
Control          & \metric{\pz5.485}{\pz2.541} & \metric{\pz0.633}{\pz0.142} & \metric{\pz0.408}{\pz0.043} & \metric{\pz0.274}{\pz0.022} \\
\bottomrule
\end{tabular}%
\end{table}

\textbf{Sample-based metrics.}
Beyond likelihood, we evaluate every model on MSE, CRPS and the Energy Score (ES). 
However, note that proper scoring rules agree only at the ground truth and rank \emph{imperfect} models differently~\citep{Gneiting2007.Strictly}. 
The ES in particular is a weak instrument for what a copula contributes, as its power to discriminate dependence decays with the dimension while its sensitivity to the marginal means does not~\citep{Marcotte2023.Regions}.
The decoupled construction makes the marginal objective a free choice, so we fit \margmodel{} on CRPS directly and leave the copula stage untouched.
The resulting marginal model attains the best CRPS and MSE on three of four datasets, and \model{} on top of it the best ES on two. 
A jointly trained model has no analogue of this. Its single objective has to cover the dependence structure as well as the marginals.
Results and details are in \Cref{app:crps-training,app:sampling-results}.

\section{Related Work}\label{sec:related}

\textbf{Neural processes.}
Many conditional neural processes~\citep{Garnelo2018.Conditional} only model univariate marginals and model the joint distribution by autoregressively adding the prediction to the context~\citep{Bruinsma2022.Autoregressive,Hassan2026.Efficient,Mortimer2026.Incremental}. While expressive, this approach sacrifices permutation equivariance and marginalization consistency. Latent neural processes~\citep{Garnelo2018.Neural,Kim2018.Attentive,Foong2020.MetaLearning} are consistent by construction but have an intractable likelihood, while Gaussian neural processes~\citep{Bruinsma2020.Gaussian} have an exact likelihood but a Gaussian joint.
Diffusion- and flow-based models~\citep{Dutordoir2023.Neural,Kerrigan2024.Functional} are more flexible, but provide no closed-form density for a given target set.

\textbf{Copulas in machine learning.}
Gaussian copula processes~\citep{Wilson2010.Copula} and Gaussian process conditional copulas~\citep{Hernandez-Lobato2013.Gaussian} are Bayesian models fit to a single dataset rather than amortized over context sets.
Neural models that do amortize the Gaussian copula over contexts~\citep{Salinas2019.Highdimensional,Markou2021.Practical} are limited to linear correlations.
Gaussian copulas suit copula processes because they are defined for any number of target points and are closed under marginalization.
Archimedean copulas~\citep{Ling2020.Deep} share these properties but are exchangeable, so all pairs of target points have the same dependence.
Vine copulas~\citep{Aas2009.Paircopula} and flow- or diffusion-based copulas~\citep{Kamthe2021.Copula,Huk2026.Diffusion} are far more flexible, but they are defined for a fixed set of variables. Implicit generative copulas~\citep{Wen2019.Deep} lack a tractable density altogether.
Stationary vine processes~\citep{Bladt2022.Time,Nagler2022.Stationary} do extend a vine to a process over
a regular time grid, fit to a single time series rather than conditioned on a context.
The attentional copula of TACTiS/TACTiS-2~\citep{Drouin2022.TACTiS,Ashok2023.TACTiS2} is autoregressive over a permutation of the target points and therefore neither marginalization consistent nor guaranteed to preserve its marginals (\Cref{tab:wd}, \Cref{app:tactis-breaker}).

\textbf{Probabilistic time series forecasting.}
Probabilistic forecasters for regular multivariate time series, such as DeepAR~\citep{Salinas2020.DeepAR}, conditional normalizing flows~\citep{Rasul2021.Multivariate}, diffusion models~\citep{Rasul2021.Autoregressive,Tashiro2021.CSDI}, and K2VAE~\citep{Wu2025.K2VAE}, forecast on a fixed grid of future time steps and cannot be queried at arbitrary times.
They thus model a conditional joint distribution over a fixed set of variables rather than a continuous-time stochastic process, and they do not directly apply to IMTS\@.
Most IMTS forecasters predict point estimates~\citep{Yalavarthi2024.GraFITi,Zhang2024.Irregular,Luo2025.HiPatch,Li2025.HyperIMTS}, and early probabilistic ones predict an independent Gaussian per target point~\citep{DeBrouwer2019.GRUODEBayes,Bilos2021.Neural,Schirmer2022.Modeling}.
Among joint IMTS models, multitask Gaussian processes~\citep{Durichen2015.Multitask} are consistent but Gaussian, and ProFITi~\citep{Yalavarthi2025.Probabilistic} is expressive but not consistent~\citep{Yalavarthi2025.Reliable}.
MOSES~\citep{Yalavarthi2025.Reliable} and CircuITS~\citep{Kloetergens2026.Probabilistic} are consistent by construction but are not copula models, so marginals and dependence have to be learned jointly.
MOSES, for instance, applies univariate flows to each component of a GMM, so with more than one component its marginals depend on all components and cannot be provided by a separate model.

\section{Conclusion}\label{sec:conclusion}

We introduced \model{}, a conditional copula process whose dependency structure is a Gaussian Mixture Copula.
Inferring the latent mixture in a target-separable way makes it marginalization consistent by construction, while letting the components keep separate means carries it beyond the elliptical dependence of Gaussian copula processes.
On four IMTS benchmarks \model{} attains a better likelihood than every jointly trained model.
The most interesting direction is to scale \model{} to tabular foundation models (TFMs)~\citep{Hollmann2025.Accurate}.
These are conditional neural processes already, but predict each cell in isolation.
Currently, TFMs are already used to model dependencies, by applying them autoregressively~\citep{Vetter2025Effort}, which is not marginalization consistent~\citep{Klotergens2026tabular}.
Future work could train a foundational version of \model{} on a large corpus of synthetic data that learns to extract a dependency structure via in-context learning.

\section*{AI Use Statement}
In this work, we used generative AI tools as general-purpose writing and coding
assistants. Specifically, they were used to polish the phrasing of text that was
drafted by the authors, and to assist with routine implementation tasks such as
plotting, table formatting, and refactoring of experiment scripts. 
AI was used to draft certain appendices, 
especially in order to describe the experimental set-up in meticulous detail.  
All AI-assisted text was reviewed and edited by the
authors, and all AI-assisted code was verified by the authors and tested against
the reference implementations of the baselines. We take responsibility for the
final content of this work.

\section*{Ethics Statement}
This work is a methodological contribution to probabilistic forecasting of
irregular multivariate time series, and is evaluated exclusively on established
public benchmarks. Three of the four benchmarks (\texttt{PhysioNet'12},
\texttt{MIMIC-III}, and \texttt{MIMIC-IV}) are de-identified clinical datasets
that we use under their respective data use agreements; we neither collect new
data from human subjects nor attempt any re-identification. The main ethical consideration is the risk of over-trusting forecasts
in high-stakes clinical settings: \model{} is evaluated retrospectively and is
not a clinical decision-support tool, and any deployment in such a domain would
require dataset-specific validation, calibration analysis, and human oversight
beyond the scope of this work. The authors have no conflicts of interest to declare.

\section*{Reproducibility Statement}
We have made the following efforts to ensure that our results can be reproduced.
The full source code, including training and evaluation scripts, is available at:
\url{https://anonymous.4open.science/r/Gaussian-Mixture-Copula-Processes-7BC1}.
The experimental
setup is documented in \Cref{app:exp-details}: dataset statistics and
preprocessing in \Cref{app:dataset-stats}, the data splits in \Cref{app:splits},
the hyperparameter search spaces and selected values in \Cref{app:hyperparams},
the training setup in \Cref{app:training}, and the definitions of all reported
metrics in \Cref{app:metrics} and \Cref{app:wd}. Our adaptation of the TACTiS-2
baseline to sparse IMTS is described in \Cref{app:tactis-adapt}.

\bibliographystyle{iclr2027_conference}
\bibliography{references}

\newpage
\appendix
\addtocontents{toc}{\protect\setcounter{tocdepth}{1}}
\renewcommand{\contentsname}{Appendix: Contents}
\tableofcontents
\clearpage

\section{Expressivity: GM-C versus Mixtures of Gaussian Copulas}\label{app:gmc-expressivity}

\Cref{sec:gmc} contrasts a Gaussian Mixture Copula (GM-C) with a \emph{mixture of Gaussian copulas} (Mix-GC) and states that the latter can never deplete mass at the center of the unit cube, regardless of the number of components. We prove that statement here as an obstruction that holds uniformly in $K$, and then read it off the two toy distributions of \Cref{fig:xo}.

\paragraph{Setup and notation.}
Both constructions use $K$ mixture weights $\pi_j$ and $K$ positive-definite matrices. They differ in where the copula transformation~\Cref{eq:copula} is applied. A Mix-GC combines $K$ Gaussian copulas \emph{after} the transformation,
\begin{equation}\label{eq:mix-gc-app}
    \cop_{\mathrm{Mix}}(\mathbf{u}) \;=\; \sum_{j=1}^{K} \pi_j\, \cop_{R_j}(\mathbf{u}),
    \qquad
    \cop_{R_j}(\mathbf{u}) \;=\; \frac{\mathcal{N}\bigl(\Phi^{-1}(\mathbf{u}); 0, R_j\bigr)}{\prod_{n=1}^{N} \varphi\bigl(\Phi^{-1}(u_n)\bigr)},
\end{equation}
where $\Phi$ denotes the standard normal CDF, $\Phi^{-1}$ its inverse (the standard normal quantile function) and $\varphi$ the standard normal density, all applied elementwise. Each $R_j$ is a correlation matrix. A GM-C combines the same components \emph{before} the transformation. It applies~\Cref{eq:copula} to the joint mixture density $\pdfgmm(\mathbf{z}) = \sum_{j=1}^{K} \pi_j \mathcal{N}(\mathbf{z}; \mu_j, \Sigma_j)$ and uses that mixture's own marginal CDFs $\cdfgmmn{n}$,
\begin{equation}\label{eq:gm-c-app}
    \cop_{\mathrm{GM}}(\mathbf{u}) \;=\; \frac{\pdfgmm\bigl(\cdfgmminv(\mathbf{u})\bigr)}{\prod_{n=1}^{N} \pdfgmmn{n}\bigl(\cdfgmmninv{n}(u_n)\bigr)},
    \qquad
    \cdfgmminv(\mathbf{u}) \;=\; {\bigl(\cdfgmmninv{n}(u_n)\bigr)}_{n=1}^{N}.
\end{equation}
Write $\mathbf{z} = \Phi^{-1}(\mathbf{u})$ for the latent point that $\mathbf{u}$ is mapped to. Then~\Cref{eq:mix-gc-app} shows that the latent density a Mix-GC places on $\R^N$ is $\sum_{j=1}^{K} \pi_j \mathcal{N}(\mathbf{z}; 0, R_j)$, a Gaussian mixture whose components are all \emph{pinned to the origin} and whose marginals are all standard normal. The free means $\mu_j$ and per-component scales of~\Cref{eq:gm-c-app} have no counterpart. 

\paragraph{An obstruction at the center.}
Recall what the \emph{value} of a copula density means. The independence copula has density $\cop \equiv 1$, so by~\Cref{eq:copula} the number $\cop(\mathbf{u})$ is the likelihood ratio of the joint against independence at $\mathbf{u}$. Above $1$, the dependence structure concentrates mass relative to independent components. Below $1$, it depletes mass. A mixture of Gaussian copulas can only concentrate mass at the center of the unit cube. It can never deplete it, whatever $K$ it is given.

\begin{proposition}\label{prop:mix-gc-center}
Let $\cop_{\mathrm{Mix}}$ be any mixture of Gaussian copulas as in~\Cref{eq:mix-gc-app}, with arbitrary $K \in \mathbb{N}$, weights $\pi_j \ge 0$ summing to $1$, and correlation matrices $R_j$. Let $\mathbf{u}^\star = (\nicefrac{1}{2}, \dots, \nicefrac{1}{2})$ denote the center of the unit cube. Then
\begin{enumerate}[label=(\roman*),topsep=2pt,itemsep=1pt]
    \item $\cop_{\mathrm{Mix}}(\mathbf{u}^\star) \;=\; \sum_{j=1}^{K} \pi_j\, {\lvert R_j \rvert}^{-1/2} \;\ge\; 1$, with equality if and only if $R_j = I$ for every $j$ with $\pi_j > 0$;
    \item $\cop_{\mathrm{Mix}}$ is radially symmetric, i.e.\ $\cop_{\mathrm{Mix}}(\mathbf{u}) = \cop_{\mathrm{Mix}}(\mathbf{1} - \mathbf{u})$ for all $\mathbf{u} \in {(0,1)}^N$~\citep{Khaled2023.Approximating}.
\end{enumerate}
No mixture of Gaussian copulas therefore represents a copula whose density at the center falls below $1$, and none represents a copula that is not radially symmetric. Both statements hold for every $K$, so adding components cannot lift the restriction.
\end{proposition}

\begin{proof}
The standard normal CDF $\Phi$ has median $0$, so $\Phi^{-1}(\nicefrac{1}{2}) = 0$ and evaluating~\Cref{eq:mix-gc-app} at $\mathbf{u}^\star$ gives $\mathbf{z} = 0$. Both the numerator and the denominator of $\cop_{R_j}$ are then Gaussian densities evaluated at their own mean,
\begin{equation}
    \cop_{R_j}(\mathbf{u}^\star)
    \;=\; \frac{\mathcal{N}(0; 0, R_j)}{\prod_{n=1}^{N} \varphi(0)}
    \;=\; \frac{{(2\pi)}^{-N/2} {\lvert R_j \rvert}^{-1/2}}{{(2\pi)}^{-N/2}}
    \;=\; {\lvert R_j \rvert}^{-1/2}.
\end{equation}
By Hadamard's inequality, $\lvert R_j \rvert \le \prod_{n} {(R_j)}_{nn} = 1$ for any correlation matrix, with equality if and only if $R_j$ is diagonal, hence equal to $I$. Therefore ${\lvert R_j \rvert}^{-1/2} \ge 1$ for every $j$, and hence so is the convex combination, which proves~(i). For~(ii), the symmetry $\Phi(-z) = 1 - \Phi(z)$ of the standard normal CDF gives $\Phi^{-1}(\mathbf{1} - \mathbf{u}) = -\Phi^{-1}(\mathbf{u})$, and both $\mathcal{N}(\,\cdot\,; 0, R_j)$ and $\varphi$ are even, so every summand in~\Cref{eq:mix-gc-app} is invariant under $\mathbf{u} \mapsto \mathbf{1} - \mathbf{u}$~\citep{Khaled2023.Approximating}.
\end{proof}

For $N = 2$ the bound is explicit. A bivariate Gaussian copula with correlation $\rho$ has $\lvert R \rvert = 1 - \rho^2$, so
\begin{equation}\label{eq:center-2d}
    \cop_{\mathrm{Mix}}(\nicefrac{1}{2}, \nicefrac{1}{2}) \;=\; \sum_{j=1}^{K} \frac{\pi_j}{\sqrt{1 - \rho_j^2}} \;\ge\; 1 ,
\end{equation}
which equals $1$ only in the independent case $\rho_j = 0$ and grows without bound as any $\lvert \rho_j \rvert \to 1$. Whatever correlations and weights are chosen, and however many components are used, the center of the unit square is at least as dense as it would be under independence.

The bound excludes target densities that have a \emph{hole} at their center. Suppose a density $f$ on $\R^N$ is symmetric about the origin in each coordinate, so that $F_n(0) = \nicefrac{1}{2}$ and the origin is mapped to $\mathbf{u}^\star$. Suppose further that $f(0) \approx 0$ while the marginals $f_n(0)$ remain strictly positive. Evaluating~\Cref{eq:copula} at $\mathbf{y} = 0$,
\begin{equation}\label{eq:center-ratio}
    \cop(\mathbf{u}^\star) \;=\; \frac{f(0)}{\prod_{n=1}^{N} f_n(0)} \;\approx\; 0 ,
\end{equation}
which falls below the bound of \Cref{prop:mix-gc-center}~(i) by the full margin. The ring of \Cref{fig:xo} is such a density. It has no mass at the origin, yet each of its marginals is strictly positive there, because the ring's mass projects onto the whole interval in every coordinate. Its copula lies outside the Mix-GC class, and no number of components brings it inside.

A GM-C obeys neither constraint of \Cref{prop:mix-gc-center}. Placing the component means $\mu_j$ away from the origin makes the marginals $\cdfgmmn{n}$ multimodal. The transformation $\cdfgmminv$ in~\Cref{eq:gm-c-app} then stretches the low-density valleys between the modes over a wide range of $\mathbf{u}$ and compresses the modes, so the resulting copula density comes arbitrarily close to $0$ at $\mathbf{u}^\star$ and is in general not symmetric about it.

\paragraph{The X distribution.}
\Cref{fig:xo}(a) is the case a Mix-GC is built for. The X density superposes two elongated Gaussian ridges with correlations $\approx +\rho$ and $\approx -\rho$ that share the origin as their common center, so both components already sit where~\Cref{eq:mix-gc-app} forces them to sit. Its marginals are unimodal. Its copula density peaks at the center and is radially symmetric, so it satisfies both conclusions of \Cref{prop:mix-gc-center}, and Mix-GC and GM-C are visually indistinguishable from the ground truth. On dependence structures of this type the cheaper Mix-GC parameterization is an adequate model, which is also why the two variants are close on some of the real datasets in \Cref{tab:njnll}.

\paragraph{The circle distribution.}
\Cref{fig:xo}(b) is the instance of~\Cref{eq:center-ratio} discussed above, and \Cref{prop:mix-gc-center} already tells us the ring is unreachable. The figure shows \emph{how} a Mix-GC fails when asked to fit it anyway. It cannot carve out the hole, and \Cref{prop:mix-gc-center}~(i) penalizes the surplus density that any $R_j \ne I$ adds at the center, so the fitted model drives its components towards the equality case $R_j \approx I$, that is, towards near-independence. Under near-independence the joint reduces to the product of the marginals. Both marginals are bimodal with modes near $\pm 2.8$, so the product has four modes at the corners $(\pm 2.8, \pm 2.8)$. That is the square-with-corner-modes pattern of the Mix-GC panel. The model reproduces the correct marginals, as any copula construction must, and then has nowhere to place the mass except the four locations the marginals alone dictate. The ring's angular coupling is inexpressible in the class.

The GM-C recovers the ring. Its latent components are not restricted to share the same center. The transformation $\cdfgmminv$ then maps the center of the unit cube into the low-density valley between the modes and opens the hole. The residual squareness is a finite-$K$ effect. 

\clearpage
\section{Marginalization Consistency of \model{}}\label{app:mc-proof}

We prove here that \model{} is marginalization consistent by construction. The argument reuses the framework of separable flows and mixtures of separable flows introduced by~\citet{Yalavarthi2025.Reliable}. We first restate the two lemmas of that paper that we will invoke, adapted to the notation of the present work, and then apply them to \model{}.

We adopt the notation of \Cref{sec:model}, in which $\Dctx$ denotes the context set, $\Xtgt = (\xtgt_1, \dots, \xtgt_N)$ the target set, $\xtgt_n$ the $n$-th target input, $\cdfmargn{n}$ the marginal CDF at target point $n$, and $\cdfgmmn{n}$ the $n$-th one-dimensional marginal of the latent Gaussian mixture density $\pdfgmm$. A conditional density $\hat{p}(\Ytgt \mid \Xtgt, \Dctx)$ on $\R^{N}$ is \emph{marginalization consistent} if it satisfies~\Cref{eq:mc-def} for every $k \in \{1, \dots, N\}$.

\subsection{Cited Lemmas from Yalavarthi et al.\ (2026)}\label{app:mc-lemmas}

\begin{lemma}[Separable flows preserve marginalization consistency; \citealp{Yalavarthi2025.Reliable}, Lemma~3.1]\label{lem:moses-31}
Let $T(\,\cdot\, \mid \Xtgt, \Dctx) : \R^{N} \to \R^{N}$ be a conditional flow that is separable, i.e.\ of the form
\begin{equation}
T(\mathbf{z} \mid \Xtgt, \Dctx) \;=\; \bigl(\tau(z_1 \mid \xtgt_1, \Dctx), \dots, \tau(z_{N} \mid \xtgt_{N}, \Dctx)\bigr)
\end{equation}
for some univariate function $\tau$ that is invertible in its first argument. If the base density $\hat{p}_Z(\mathbf{z} \mid \Xtgt, \Dctx)$ is marginalization consistent, then the pushforward density
\begin{equation}
\hat{p}(\Ytgt \mid \Xtgt, \Dctx) \;=\; \hat{p}_Z\bigl(T^{-1}(\Ytgt \mid \Xtgt, \Dctx) \,\big|\, \Xtgt, \Dctx\bigr) \cdot \left| \det \frac{\partial T^{-1}(\Ytgt \mid \Xtgt, \Dctx)}{\partial \Ytgt} \right|
\end{equation}
is also marginalization consistent.
\end{lemma}

\begin{lemma}[Mixtures with target-independent weights preserve marginalization consistency; \citealp{Yalavarthi2025.Reliable}, Lemma~3.2]\label{lem:moses-32}
Let $\hat{p}_1, \dots, \hat{p}_D$ be conditional densities that are each marginalization consistent, and let $w$ be a weight function mapping a context set to the simplex $\Delta^D$, i.e.\ depending only on $\Dctx$. Then the mixture
\begin{equation}
\hat{p}(\Ytgt \mid \Xtgt, \Dctx) \;=\; \sum_{d=1}^{D} w_d(\Dctx) \, \hat{p}_d(\Ytgt \mid \Xtgt, \Dctx)
\end{equation}
is marginalization consistent.
\end{lemma}

For completeness, we recall the proof idea of \Cref{lem:moses-31}: separability makes the Jacobian of $T^{-1}$ diagonal, so its determinant factors across coordinates, the integral over $\ytgt_k$ collapses to an integral over $z_k$ via the change-of-variables theorem, and marginalization consistency of $\hat{p}_Z$ then yields the claim. \Cref{lem:moses-32} follows by interchanging the integral over $\ytgt_k$ with the finite sum over mixture components, which is permitted because $w_d$ does not depend on $\Xtgt$. Full proofs are given in Appendices~A.1 and~A.2 of \citet{Yalavarthi2025.Reliable}.

\subsection{Main Result}\label{app:mc-main}

We restate the result of \Cref{sec:model} and prove it.

\begin{theoremrestate}[\Cref{thm:mc}, restated]
\model{}, paired with any model that models univariate marginals independently, is marginalization consistent. Hence, for every context set $\Dctx$ and every target set $\Xtgt$, the predicted joint density satisfies marginalization consistency as defined in ~\Cref{eq:mc-def}.
\end{theoremrestate}

\begin{proof}[Proof of \Cref{thm:mc}]
The proof proceeds in three steps.

\paragraph{Step 1: \model{} is a target-separable transformation of a latent base density.}
The generative process of \model{} maps a latent vector $\mathbf{z}$, drawn from the latent Gaussian mixture base density $\pdfgmm(\mathbf{z} \mid \Xtgt, \Dctx)$, to the output domain via the per-coordinate transformation
\begin{equation}
\ytgt_n \;=\; \cdfmargninv{n}\bigl(\cdfgmmn{n}(z_n)\bigr).
\end{equation}
Both $\cdfmargn{n}$ and $\cdfgmmn{n}$ are parameterized by quantities that, by target separability, depend only on $\Dctx$ and on the single target input $\xtgt_n$. Defining $\tau(z_n \mid \xtgt_n, \Dctx) \coloneqq \cdfmargninv{n}(\cdfgmmn{n}(z_n))$, the joint transformation has the separable form required by \Cref{lem:moses-31}. By that lemma, marginalization consistency of \model{} reduces to marginalization consistency of the latent base density $\pdfgmm$.

\paragraph{Step 2: Each Gaussian component of the latent mixture is marginalization consistent.}
The latent base density is the Gaussian mixture
\begin{equation}
\pdfgmm(\mathbf{z} \mid \Xtgt, \Dctx) \;=\; \sum_{j=1}^{K} \pi_j \, \mathcal{N}(\mathbf{z}; \mu_j, \Sigma_j).
\end{equation}
As established in \Cref{sec:model}, the entry ${[\mu_j]}_n$ depends only on $\xtgt_n$ and $\Dctx$ (see \Cref{eq:mu-sigma}), and the entry $\Sigma_{j, nm}$ only on $\xtgt_n$, $\xtgt_m$ and $\Dctx$ (see \Cref{eq:mu-sigma,eq:V,eq:sigma-j}).
Removing a subset of target points from $\Xtgt$ therefore deletes the corresponding entries of $\mu_j$ and the corresponding rows and columns of $\Sigma_j$, leaving the remaining entries unchanged. This coincides with the standard marginalization rule for a multivariate Gaussian, so each component $\mathcal{N}(\mathbf{z}; \mu_j, \Sigma_j)$ is marginalization consistent in the sense of~\Cref{eq:mc-def}.

\paragraph{Step 3: The mixture weights are target-independent, so the latent GMM is marginalization consistent.}
The mixture weights $\pi = (\pi_1, \dots, \pi_K)$ are inferred from the context set $\Dctx$ alone and are uninformed of the target set $\Xtgt$.
Together with Step~2, the two hypotheses of \Cref{lem:moses-32} are satisfied, so $\pdfgmm$ is marginalization consistent.

Combining Steps~1 to 3, the latent base density of \model{} is marginalization consistent and is pushed forward to the output domain by a target-separable transformation.
\Cref{lem:moses-31} then yields that the full joint density predicted by \model{} is marginalization consistent.
\end{proof}

The IMTS instantiation of \Cref{sec:imts} satisfies the hypotheses. The CircuITS encoder is target separable, and \margmodel{} conditions its flow parameters on $\e_n$ alone (\Cref{app:margflow}).

Step~3 is the only place where \model{} gives up a degree of freedom it could otherwise have taken. \Cref{app:target-weights} ablates that step, letting the mixture weights read the target set, and reports what the guarantee costs in likelihood.

\clearpage
\section{On TACTiS-2's Marginalization Consistency}\label{app:tactis-breaker}

The Wasserstein gap of \Cref{tab:wd} probes only the \emph{univariate} aspect of marginalization consistency, whether the univariate marginals of the model's samples still match \margmodel{}'s. We put this in the main paper because uniform univariate marginals are the defining property of a copula, and any deviation already violates the Sklar decomposition the model is built on. The full form is stronger. It asks that the model produce the \emph{same} marginal on \emph{any} subset of dimensions, regardless of which other dimensions are requested alongside. \model{} satisfies that condition by construction (\Cref{app:mc-proof}). TACTiS-2's attentional copula guarantees neither the univariate nor the multivariate version, since every target subset induces a different autoregressive factorization and attention conditioning. Below we isolate the multivariate failure on a controlled 3-d toy distribution. TACTiS-2's bivariate marginal requested directly disagrees visibly with the same bivariate marginal obtained by marginalizing its 3-d joint, while \model{}'s two routes coincide up to sampling noise.

\begin{figure}[h]
    \includegraphics[width=\textwidth]{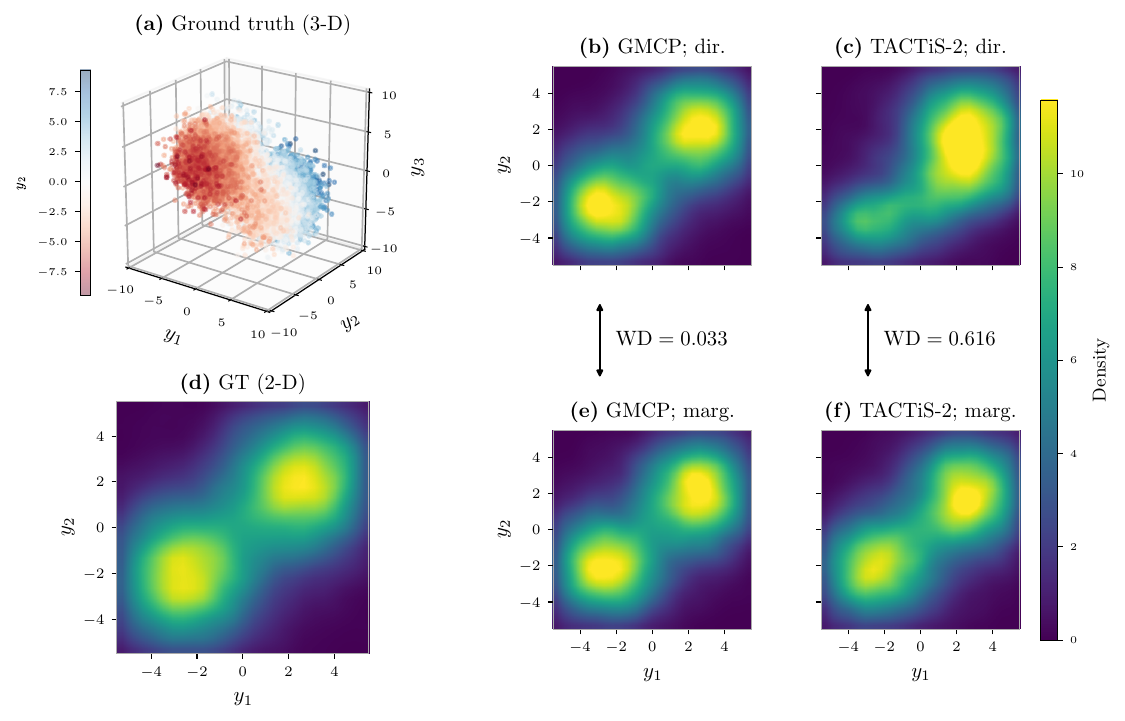}
    \caption{The Attentional Copula fails to be marginalization consistent. (a) shows the ground truth 3-d distribution. Panels (b)~and~(e) show \model{}'s direct and marginalized views of $p(y_1, y_2)$, which match by construction. Panels (c)~and~(f) show TACTiS-2's direct and marginalized views, which disagree visibly; the arrows report the corresponding WD.}\label{fig:3d-toy}
\end{figure}

\paragraph{Setup.}
We construct a 3-dimensional toy distribution with two symmetric Gaussian clusters indexed by a quasi-discrete cluster variable $y_2$. Cluster A is centered at $(y_1, y_2, y_3) \approx (-2.5, -2, +2.5)$, and cluster B at $(y_1, y_2, y_3) \approx (+2.5, +2, -2.5)$. Knowing $y_2$ localizes $y_1$, so the true 2-d marginal $p(y_1, y_2)$ is unimodal per cluster, while the 2-d marginal $p(y_1, y_3)$ is a strongly bimodal anti-diagonal mixture. We train both models on the full 3-d distribution. Both models share the same pre-trained 1-d marginals $p(y_i)$, so the only component that differs between them is the copula. Any disagreement in the predicted joints is therefore attributable to the copula alone.

For each model, we obtain the 2-d marginals $p(y_1, y_2)$ by sampling in two different ways:
\begin{enumerate}
\item \emph{Direct:} request only the dimensions $(y_1, y_2)$, with the third dimension absent from the target set.
\item \emph{Marginalized:} request the full 3-d distribution $(y_1, y_2, y_3)$ and marginalize $y_3$ out.
\end{enumerate}
A marginalization consistent copula must produce the same distribution under both routes. We quantify the gap with the Wasserstein distance (WD) between the two sample sets, indicated by the arrows between the two rows of \Cref{fig:3d-toy}.

\paragraph{Why TACTiS-2 fails.}
The attentional copula factorizes the joint as an autoregressive product of conditional CDFs over the target dimensions, with attention conditioned on the specific index set $S$ being modeled,
\begin{equation}
\hat{p}(\mathbf{y}_S) \;=\; \prod_{i \in S} \hat{p}\bigl(y_i \,\big|\, y_{<i}, S\bigr).
\end{equation}
Both the autoregressive ordering and the attention weights depend on $S$, so each target set instantiates a distinct factorization of the same network. During training on the full 3-d distribution the model only ever observes $S = \{1, 2, 3\}$, so the only factors that receive a gradient signal are those induced by that target set, namely $\hat{p}(y_3)$, $\hat{p}(y_2 \mid y_3)$, and $\hat{p}(y_1 \mid y_2, y_3)$. Evaluating the bivariate marginal under the direct route asks the same network for a different set of factors, $\hat{p}(y_2)$ and $\hat{p}(y_1 \mid y_2)$, which are never instantiated during training. The attentional copula is free to map this out-of-distribution conditioning to an arbitrary distribution, and no architectural mechanism ties its output back to the marginal of the trained 3-d joint. 

\paragraph{Multivariate consistency on the benchmark datasets.}
The experiment above isolates the failure on a toy distribution. We now check that it also occurs on the four IMTS datasets, again comparing a direct and a marginalized route, but now on the multivariate marginal $\hat p(\Ytgt_{-k} \mid \Xtgt_{-k}, \Dctx)$ of \Cref{eq:mc-def} instead of a univariate one. For each test instance we remove one target point $k$ from the target set. The \emph{direct} route samples the model on the reduced target set $\Xtgt_{-k}$. The \emph{marginalized} route samples it on the full target set $\Xtgt$ and discards the $k$-th coordinate. Since the two sample sets live on $\R^{N-1}$, we compare them with the energy distance, which for distributions $P$ and $Q$ with independent draws $\mathbf{a}, \mathbf{a}' \sim P$ and $\mathbf{b}, \mathbf{b}' \sim Q$ is
\begin{equation}
D_E(P, Q) \;=\; 2\,\mathbb{E}\,\lVert \mathbf{a} - \mathbf{b} \rVert \;-\; \mathbb{E}\,\lVert \mathbf{a} - \mathbf{a}' \rVert \;-\; \mathbb{E}\,\lVert \mathbf{b} - \mathbf{b}' \rVert,
\end{equation}
estimated with sample averages. As in \Cref{tab:wd}, a finite sample never yields a distance of exactly zero, so we normalize by a control: the energy distance between two sample sets of the direct route drawn under different random seeds. The resulting ratio
\begin{equation}
\rho \;=\; \frac{D_E\bigl(\text{marginalized},\, \text{direct}\bigr)}{D_E\bigl(\text{direct (seed 1)},\, \text{direct (seed 2)}\bigr)}
\end{equation}
expresses the inconsistency in units of the sampling error. A marginalization consistent model has $\rho \approx 1$, since removing a target point then changes its prediction no more than reseeding does. \Cref{tab:subset-mc} confirms that \model{} stays close to $1$ on all datasets, whereas TACTiS-2 exceeds the sampling error by a factor of $11$ to $70$.

\begin{table}[h]
\centering
\caption{Multivariate marginalization consistency. Ratio $\rho$ of the energy distance between the marginalized and the direct route, with one target point removed, to the energy distance between two seeds of the direct route. $\rho \approx 1$ means the inconsistency is at the level of the sampling error. $^\dagger$Our adaptation of TACTiS-2 to sparse IMTS (\Cref{app:tactis-adapt}).}\label{tab:subset-mc}
\small
\setlength{\tabcolsep}{3mm}
\begin{tabular}{lcccc}
\toprule
\textbf{Model} & \multicolumn{1}{c}{\texttt{USHCN}} & \multicolumn{1}{c}{\texttt{PhysioNet'12}} & \multicolumn{1}{c}{\texttt{MIMIC-III}} & \multicolumn{1}{c}{\texttt{MIMIC-IV}} \\
\midrule
\tactis{}         & 28.963 & 53.588 & 70.500 & 11.101 \\
\model{} (GM-C)  & \pz1.176 & \pz1.035 & \pz1.069 & \pz1.060 \\
\bottomrule
\end{tabular}
\end{table}

\clearpage
\section{Model Architecture}\label{app:arch}

\Cref{fig:arch} shows how \margmodel{} (\Cref{app:margflow}) and \model{} (\Cref{sec:model}) work together to model joint distributions of IMTS forecasting targets. The encoder of \Cref{app:encoder} supplies the target encodings $\e_n$ and the context summary $\bar{\mathbf{h}}$ to both.
\begin{figure}[h]
  \centering
  \resizebox{\textwidth}{!}{\begin{tikzpicture}[
    >=Stealth, %
    font=\sffamily\small, %
    box/.style ={rounded corners=3pt, align=center, fill=white,
                 inner sep=4pt, minimum height=6mm, minimum width=12mm},
    enc/.style ={box, draw=margline, semithick, fill=margnode},
    flow/.style={box, draw=margline, semithick, fill=margnode},
    gmm/.style ={box, draw=copline, semithick, fill=copnode},
    comp/.style={box, fill=brandred, minimum width=22mm, minimum height=5mm},
    sumb/.style={box, fill=brandred, minimum width=26mm},
    icdf/.style={box, fill=brandred, minimum width=18mm},
    sbox/.style={box, fill=brandorange},
    data/.style={align=center, font=\sffamily\scriptsize, inner sep=1pt},
    panelGMM/.style ={draw=copline!40, densely dotted, rounded corners=4pt, fill=copnode!40}, 
    fwd/.style ={->, semithick, draw=black!60}, %
    den/.style ={->, thick, draw=jointline},     %
    smp/.style ={->, thick, draw=sampline, densely dashed}, %
    node distance=4mm and 5mm,
]

\node[data] (X) {\large $\Dctx$};
\node[data, below=1cm of X] (Q) {\large $\Xtgt$};

\node[enc,  above right=10mm and 1cm of X] (encM) {Enc$_{\text{m}}$};
\node[flow, right=2cm of encM] (dsf) {DSF};
\node[enc,  below=of dsf] (u) {$u_n=\cdfmargn{n}(\ytgt_n)$};

\node[data, above=1cm of dsf] (Y) {\large $\ytgt_n$};

\node[gmm,  below right=10mm and 1cm of Q] (encC) {Enc$_{\text{c}}$};
\node[gmm,  right=of encC, minimum width=18mm] (mus) {$\Sigma_1, \ldots, \Sigma_K$};
\node[gmm,  below=of mus, minimum width=18mm] (sigma) {$\mu_1, \ldots, \mu_K$};
\node[gmm,  above=of mus, minimum width=18mm] (pi) {$\pi$};

\node[gmm, above right=1mm and 8mm of mus, fill=sampnode, draw=sampline] (zhat) {$\mathbf{\hat{z}}$};
\node[gmm, right=1.1cm of zhat, minimum width=22mm, fill=sampnode, draw=sampline] (gmmcdf)
    {$\hat{u}_n = \cdfgmmn{n}(\hat{z}_n)$};

\node[gmm, below right=1mm and 12mm of mus] (Finv) {$z_n=\cdfgmmninv{n}(u_n)$};
\node[gmm, right=of Finv] (cop) {$\frac{\pdfgmm(\mathbf{z})}{\prod_n^N \pdfgmmn{n}(z_n)}$};
\node[enc] (marg-lik) at (cop|-dsf) {$\prod^N_n \pdfmargn{n}(\ytgt_n)$};
\node[data, above=1mm of marg-lik] (marg-lik-text) {\textsc{likelihood}};
\node[data, above=0mm of marg-lik-text] {\textsc{marginal}};

\begin{scope}[on background layer]
  \path let \p1=(encM.west), \p2=(encC.west) in
        coordinate (panelLeft)  at (\x1,0) coordinate (panelLeft2)  at (\x2,0);
  \path let \p1=(marg-lik-text.east), \p2=(gmmcdf.east), \p3=(cop.east) in
        coordinate (panelRight) at ({max(\x1,max(\x2,\x3))},0);
  \path let \p1=(encM.west), \p2=(encC.west) in
        coordinate (panelLeftAll) at ({min(\x1,\x2)},0);

  \node[draw=black!25, densely dashed, rounded corners=6pt, fill=black!2,
        fit=(encM)(dsf)(u)(marg-lik-text)(panelLeftAll|-encM)(panelRight|-encM),
        inner xsep=5pt, inner ysep=4pt,
        label={[anchor=north west, xshift=2pt, font=\sffamily\bfseries\color{black!70}]north west:\margmodel{}}]
       (margpanel) {};
       
  \node[draw=black!25, densely dashed, rounded corners=6pt, fill=copbg!40,
      fit=(encC)(sigma)(pi)(mus)(cop)(zhat)(gmmcdf)(panelLeftAll|-encC)(panelRight|-encC),
        inner xsep=5pt, inner ysep=5mm,
        label={[anchor=north west, xshift=2pt, font=\sffamily\bfseries\color{black!70}]north west:\model{}}]
       (coppanel) {};
       
  \node[panelGMM,
      fit=(pi)(sigma)(mus),
      inner xsep=4pt, inner ysep=3mm,
      label={[text=copline, font=\sffamily\scriptsize\bfseries, anchor=north, yshift=2pt]north: GMM}]
     (gmmpanel) {};
\end{scope}

\draw[fwd] (X) -- (encM.west);
\draw[fwd] (Q) -- (encC.west);
\draw[fwd] (X.east) -- (encC.west);
\draw[fwd] (Q.east) -- (encM.west);
\draw[fwd] (encM) -- (dsf);
\draw[den] (Y) -- (dsf);
\draw[den] (dsf) -- (u);
\draw[den] (dsf) -- (marg-lik);
\draw[fwd] (encC) -- (mus);
\draw[fwd] (encC) -- (pi.west);
\draw[fwd] (encC) -- (sigma.west);
\draw[den] (gmmpanel.east) -- (Finv.west);

\node[data, above=1mm of cop] {\textsc{copula density}};

\draw[den] (u.south) -- ($(u.south)+(0,-0.75cm)$) -| (Finv.north);

\node[circle, draw=brandblueborder, semithick, inner sep=0pt, minimum size=6mm, fill=brandblue]
    (mult) at ([xshift=1cm]$(margpanel.south east)!0.5!(coppanel.north east)$) {$\mathbf{\times}$};
    
\node[box, draw=brandblueborder, semithick, fill=brandblue!30, minimum width=18mm, right=of mult] (pY) {$p(\Ytgt \mid \Xtgt, \Dctx)$};

\node[data, above=1mm of pY] (py-text) {\textsc{joint likelihood}};
\node[data, above= 0mm of py-text] (condi) {\textsc{conditional}};

\draw[den] (Finv) -- (cop);
\draw[den] (marg-lik.east) -| (mult.north);
\draw[den] (cop.east) -| (mult.south);
\draw[den] (mult.east) -- (pY.west);

\node[data, right=10mm of Y] (yhat) {\large $\hat{y}^{\tgt}_n$};

\draw[smp] (gmmpanel.east) -- (zhat.west);
\draw[smp] (zhat) -- (gmmcdf);
\draw[smp] (gmmcdf.north) -- ([yshift=2.5cm]gmmcdf.north) -- (dsf.south east);
\draw[smp] (dsf.north east) -- (yhat);

\end{tikzpicture}}
  \caption{\model{} architecture, instantiated for IMTS with \margmodel{} as the marginal model.
    {\color{jointline}$\boldsymbol{\rightarrow}$} is used to indicate the computation of the joint likelihood.
    In contrast, we use {\color{sampline}$\boldsymbol{\dashrightarrow}$} to indicate the sampling procedure.
    {\color{black!70}$\boldsymbol{\rightarrow}$} represents the conditioning via the encoder, which is performed during likelihood estimation and sampling.}\label{fig:arch}
\end{figure}

\clearpage
\section{The Encoder}\label{app:encoder}

Both \model{} and \margmodel{} draw their target encodings $\e_n$ and the global context summary $\bar{\mathbf{h}}$ from the encoder of CircuITS~\citep{Kloetergens2026.Probabilistic}, which we adopt unchanged. We restate it here so that the present paper is self-contained, and because the property it supplies, \emph{target separability}, is the hypothesis on which \Cref{thm:mc} rests. The single addition we make, the pooling step of~\Cref{eq:enc-pool}, is marked as such.

Throughout, $C$ is the number of channels, $D$ the latent dimension, $M$ the number of context observations, and $\mathrm{MHA}(Q, K, V)$ a multi-head attention layer~\citep{Vaswani2017.Attention} with queries $Q$, keys $K$ and values $V$. We write $\psi(\cdot)$ for a sinusoidal time embedding~\citep{Kazemi2019.Time2Vec} and $[\,\cdot \,\|\, \cdot\,]$ for concatenation. To avoid collision with the parameter functions $\phi_\pi, \phi_\mu, \phi_\sigma, \phi_V$ of \Cref{sec:model}, the time embedding is denoted $\psi$ here rather than $\phi$.

\paragraph{Step 1: Observation embeddings.}
The context set $\Dctx = {\{((t_m, c_m), \yctx_m)\}}_{m=1}^{M}$ is an unordered collection of triplets, each consisting of a continuous timestamp, a channel index and a scalar value. Every triplet is embedded independently,
\begin{equation}
    \mathbf{h}_m \;=\; \bigl[\psi(t_m) \,\|\, \mathbf{E}^{\mathrm{chan}}_{c_m} \,\|\, \yctx_m\bigr] \;\in\; \R^{D},
    \label{eq:enc-obs}
\end{equation}
where $\mathbf{E}^{\mathrm{chan}} \in \R^{C \times D'}$ is a learnable channel-embedding matrix and the three parts are sized so that the concatenation has dimension $D$. Stacking the $M$ embeddings gives $\mathbf{H}^{\mathrm{obs}} \in \R^{M \times D}$. No positional information enters beyond $\psi(t_m)$, so the representation is permutation invariant in the observations, as the set-valued nature of an IMTS requires.

\paragraph{Step 2: Intra-channel aggregation.}
A channel may carry anywhere between zero and several hundred observations, so the per-channel summaries are formed by attention rather than by a fixed-size reduction. A learnable prototype matrix $\mathbf{E}^{\mathrm{proto}} \in \R^{C \times D}$, holding one row per channel, supplies the queries of a masked cross-attention layer over the observation embeddings,
\begin{equation}
    \mathbf{H}^{\mathrm{chan}} \;=\; \mathrm{MHA}\bigl(\mathbf{E}^{\mathrm{proto}}, \mathbf{H}^{\mathrm{obs}}, \mathbf{H}^{\mathrm{obs}}; \mathbf{M}\bigr) \;\in\; \R^{C \times D},
    \label{eq:enc-intra}
\end{equation}
with a binary mask $\mathbf{M}$ whose entry $\mathbf{M}_{c,m}$ is $1$ exactly when observation $m$ belongs to channel $c$, and $0$ otherwise. The mask makes row $c$ of $\mathbf{H}^{\mathrm{chan}}$ a summary of the observations of channel $c$ alone. The channel structure of the IMTS thus stays present in the latent state instead of being flattened into a single set of triplets, which is why we prefer this encoder to a generic set transformer in the sparse regime.

\paragraph{Step 3: Inter-channel interaction.}
Channels are not independent, and a channel observed only a handful of times, or not at all, must be able to borrow information from the others. A residual self-attention layer over the channel summaries provides that,
\begin{equation}
    \tilde{\mathbf{H}} \;=\; \mathbf{H}^{\mathrm{chan}} + \mathrm{MHA}\bigl(\mathbf{H}^{\mathrm{chan}}, \mathbf{H}^{\mathrm{chan}}, \mathbf{H}^{\mathrm{chan}}\bigr) \;\in\; \R^{C \times D}.
    \label{eq:enc-inter}
\end{equation}
If channel $c$ is entirely unobserved its mask row in~\Cref{eq:enc-intra} is empty, and $\tilde{\mathbf{h}}_c$ is then determined by its prototype together with its interactions with the remaining channels.

\paragraph{Step 4: Target encodings.}
A target input is a pair $\xtgt_n = (t_n, c_n)$. Its encoding retrieves the row of $\tilde{\mathbf{H}}$ belonging to its channel, concatenates the embedded query time, and projects back to $\R^{D}$ with a single linear layer,
\begin{equation}
    \e_n \;=\; \enc\bigl(\xtgt_n, \Dctx\bigr) \;=\; \mathrm{Proj}^{\mathrm{QRY}}\bigl(\bigl[\psi(t_n) \,\|\, \tilde{\mathbf{h}}_{c_n}\bigr]\bigr) \;\in\; \R^{D}.
    \label{eq:enc-target}
\end{equation}
This is the vector that \Cref{sec:imts} feeds to $\phi_\mu$, $\phi_\sigma$ and $\phi_V$, and that \margmodel{} feeds to $\mathrm{MLP}_{\mathrm{DSF}}$.

\paragraph{Step 5: Global summary.}
\model{} additionally requires a representation of the context that carries no target information, because the mixture weights $\pi = \phi_\pi(\Dctx)$ of \Cref{sec:model} must not see $\Xtgt$. We obtain it by pooling $\tilde{\mathbf{H}}$ with a single learnable attention query $\mathbf{q} \in \R^{D}$,
\begin{equation}
    \bar{\mathbf{h}} \;=\; \mathrm{Pool}\bigl(\tilde{\mathbf{H}}\bigr) \;=\; \sum_{c=1}^{C} \alpha_c \, \tilde{\mathbf{h}}_c,
    \qquad
    \alpha = \mathrm{softmax}\bigl(\tilde{\mathbf{H}} \mathbf{q}\bigr) \in \Delta^{C},
    \label{eq:enc-pool}
\end{equation}
which is the one step that is ours rather than adopted. CircuITS instead derives its circuit weights by cross-attention with a separate set of learnable embeddings. Everything up to and including~\Cref{eq:enc-target} is taken over unchanged. \Cref{app:target-weights} ablates this pooling step by replacing $\tilde{\mathbf{H}}$ with the target encodings, which is what target separability forbids.

\paragraph{Why the encoder is target separable.}
\Cref{thm:mc} requires that dropping a target point from~$\Xtgt$ leave the encodings of the remaining target points untouched. The encoder satisfies this because no step lets one target point see another. \Crefrange{eq:enc-obs}{eq:enc-inter} read the context set only, so $\tilde{\mathbf{H}}$, and with it $\bar{\mathbf{h}}$, is a function of $\Dctx$ alone. \Cref{eq:enc-target} then combines $\tilde{\mathbf{H}}$ with a single target input. Hence
\begin{equation}
    \e_n \text{ depends on } \Xtgt \text{ only through } \xtgt_n,
    \qquad
    \bar{\mathbf{h}} \text{ does not depend on } \Xtgt \text{ at all},
\end{equation}
which are the two hypotheses invoked in Steps~1 and~3 of the proof of \Cref{thm:mc}. A Transformer Neural Process~\citep{Nguyen2022.Transformer} attends over the target set, and so does the attentional copula of TACTiS-2~\citep{Ashok2023.TACTiS2}. That destroys the property. It is why \margmodel{} performs no attention among target points (\Cref{app:margflow}), and why TACTiS-2 forfeits marginalization consistency (\Cref{app:tactis-breaker}).

\clearpage
\section{\texorpdfstring{\margmodel{}}{MargFlow}: A Conditional Neural Process with Normalizing-Flow Marginals}\label{app:margflow}

A copula-based model needs a dedicated model for the univariate marginals, and the quality of that model turns out to matter as much as the copula on top of it (\Cref{sec:experiments}). Existing marginal models for IMTS~\citep{DeBrouwer2019.GRUODEBayes,Bilos2021.Neural,Schirmer2022.Modeling} rely on Neural ODEs, which have been shown to be ineffective in forecasting~\citep{Klotergens2024.PhysiomeODE}, and restrict the predictive distribution at each target point to be Gaussian. \margmodel{} replaces them. It pairs the target-separable encoder of \Cref{sec:imts} with Deep Sigmoidal Flows (DSF)~\citep{Huang2018.Neural} as the per-target output head. Both parts and their combination are adopted from CircuITS~\citep{Kloetergens2026.Probabilistic}. Ours is the use of that combination as a model on its own, which the paragraph \emph{Relation to CircuITS} below makes precise.

\paragraph{Relation to neural processes.}
Structurally, \margmodel{} is an instance of the Conditional Neural Process recipe~\citep{Garnelo2018.Conditional}. A permutation-invariant encoder maps the context, here the observed history $\Dctx$, to a representation, and a decoder turns that representation together with a target location into a predictive distribution at that location. Instantiating that encoder with attention places \margmodel{} close to the Transformer Neural Process family~\citep{Nguyen2022.Transformer}. Two things are specialized relative to that recipe, both of them already so in CircuITS:
\begin{itemize}[topsep=2pt,itemsep=2pt]
    \item \textbf{The encoder.} Instead of a generic set transformer over context points, we use the encoder of CircuITS~\citep{Kloetergens2026.Probabilistic}, which is built for the sparse IMTS regime in which a channel may be observed only a handful of times over the whole observation window, and which attends over the channel--time structure of $\Dctx$ rather than over a flat set. It is adopted unchanged, and restated in full in \Cref{app:encoder}. %
    \item \textbf{The output head.} The usual NP decoder emits the mean and variance of a Gaussian. We replace it by a one-dimensional normalizing flow. That makes the per-target predictive distribution non-Gaussian and keeps its CDF available in closed form, which the copula construction of \Cref{sec:model} depends on, since it consumes $u_n = \cdfmargn{n}(\ytgt_n)$ directly.
\end{itemize}

\paragraph{Relation to CircuITS.}
Pairing this encoder with DSF outputs is not itself a new construction. The leaves of
CircuITS~\citep{Kloetergens2026.Probabilistic} are built from the same encoder and from normalizing-flow
marginals of the same form, and \margmodel{} is such a leaf with a single mixture component and the leaf
copula removed. CircuITS never uses that construction on its own. Its leaves exist only inside the circuit and are
fitted under the joint objective, and none of them is ever evaluated as a marginal forecaster. Our claim
here concerns the isolation rather than the architecture, and the numbers reward the isolation.
\Cref{tab:mnll} puts \margmodel{} next to CircuITS, whose induced univariate marginals are a mixture of
such flows and whose marginal family is therefore at least as rich as \margmodel{}'s. Only the
objective and the surrounding model differ, and \margmodel{} attains the lower mNLL on all four datasets.

\paragraph{Normalizing flows for the marginal CDFs.}
We use the target encodings $\e_n$ of~\Cref{eq:enc} to model the respective marginal CDFs $\cdfmargn{n}$ via a DSF~\citep{Huang2018.Neural}.
A DSF is a strictly monotone scalar transformation $T_{\theta_n}: \R \to (0, 1)$ parameterized by
$\theta_n = {\{a^{(\ell)}, b^{(\ell)}, w^{(\ell)}\}}_{\ell=1}^{L}$, defined as a composition of $L$ sigmoidal blocks.
A single block of width $W$ acts on a scalar input $x$ as
\begin{equation}
    s(x; a, b, w) = \sigma^{-1}\!\left(\sum_{i=1}^{W} w_i\, \sigma(a_i x + b_i)\right)
\end{equation}
Here $\sigma$ is the logistic sigmoid, $a_i > 0$, $b_i \in \R$, and $w_i \geq 0$ with $\sum_i w_i = 1$. Positivity of $a_i$ and the convex combination over sigmoids ensure that the inner sum is a strictly increasing function valued in $(0, 1)$, so applying $\sigma^{-1}$ yields a strictly increasing map $\R \to \R$\@.
Stacking $L$ such blocks gives a strictly increasing map $\R \to \R$, and a final logistic sigmoid squashes the output to $(0, 1)$, producing a valid CDF\@. The conditional CDF at target point $n$ is then $\cdfmargn{n}(\ytgt_n) = T_{\theta_n}(\ytgt_n)$, where the block parameters $\theta_n$ are produced from the target encoding $\e_n$ via a small MLP,
\begin{equation}
    \theta_n \;=\; \mathrm{MLP}_{\mathrm{DSF}}(\e_n).
\end{equation}
The positivity of $a^{(\ell)}$ is enforced by a softplus activation and the simplex constraint on $w^{(\ell)}$ is enforced by a softmax. Because the map is strictly monotonic, $T_{\theta_n}$ is a one-dimensional normalizing flow that pushes the target $\ytgt_n$ forward to $u_n = \cdfmargn{n}(\ytgt_n) \in (0,1)$, distributed uniformly under the model. The likelihood of a normalizing flow is obtained by the change of variable formula,
\begin{equation}
    \pdfmargn{n}(\ytgt_n) \;=\; \underbrace{p_U\bigl(T_{\theta_n}(\ytgt_n)\bigr)}_{=1} \;\cdot\; \left|\frac{\partial T_{\theta_n}}{\partial \ytgt_n}(\ytgt_n)\right|,
\end{equation}
where the first term equals one because the base distribution $p_U$ is uniform on $(0,1)$, so the likelihood reduces to the Jacobian of $T_{\theta_n}$. Conditioning $\theta_n$ only on $\e_n$ is what makes \margmodel{} target-separable, and thus marginalization consistent.

\paragraph{Training.}
\margmodel{} is trained on its own, by minimizing the marginal negative log-likelihood $\sum_n -\log \pdfmargn{n}(\ytgt_n)$, which is the first stage of the two-stage scheme of \Cref{app:training}. Its weights are then frozen and the copula is trained on top of them. \Cref{app:learning-curves} contrasts this with training all weights jointly under njNLL, and \Cref{tab:hp-margflow} lists the hyperparameter search space.

\clearpage
\section{Numerical Inversion of the GMM Marginal CDF}\label{app:numerical-icdf}

The forward pass of the copula training stage (\Cref{sec:model})
evaluates $z_n = \cdfgmmninv{n}(u_n)$.
The marginal CDF of a one-dimensional Gaussian mixture,
\begin{equation}
    \cdfgmmn{n}(z)
    \;=\; \sum_{j=1}^{K} \pi_j\,
    \Phi\!\left(\tfrac{z-\mu_{j,n}}{\sigma_{j,n}}\right),
    \label{eq:gmm-cdf}
\end{equation}
is strictly increasing and smooth in $z$, but its inverse has no closed
form. We therefore solve $\cdfgmmn{n}(z) = u$ for $z$ numerically.
The key requirements are (i) batched evaluation across all target points
and mini-batch elements simultaneously on GPU, (ii) robustness in the
extreme tails where the PDF is near zero, and (iii) sufficient accuracy
so that the analytic gradients of \Cref{prop:icdf-grad}, which are
evaluated at the returned $z$, are reliable.
A pure Newton iteration meets (i) and (iii) when initialized well.
It can also overshoot or stall in the tails, where a pure bisection
stays robust but converges only linearly. We combine the two into a
\emph{safeguarded Newton-bisection} solver, which is a standard choice
for monotone scalar root finding~\citep{Press2007.Numerical}.

\paragraph{Initial bracket.}
Because each component is Gaussian, the probability mass of
$\cdfgmmn{n}$ outside the interval
\begin{equation}
    \big[\,\min_j \mu_{j,n} - 10\,\max_j \sigma_{j,n},\;
          \max_j \mu_{j,n} + 10\,\max_j \sigma_{j,n}\,\big]
    \label{eq:icdf-bracket}
\end{equation}
is below $\Phi(-10) \approx 7.6 \times 10^{-24}$, well under
floating-point precision. We use~\Cref{eq:icdf-bracket} as the initial
bracket $[\textsc{low},\textsc{high}]$ for every $u\in(0,1)$ and
initialize the iterate at its midpoint $x \leftarrow (\textsc{low}+\textsc{high})/2$.
Throughout the iteration we maintain the invariant
$\cdfgmmn{n}(\textsc{low}) \le u \le \cdfgmmn{n}(\textsc{high})$,
so the true root is always contained in the bracket.

\paragraph{Iteration.}
At each step we evaluate the CDF and PDF at the current iterate $x$,
both available in closed form from the GMM parameters, and perform the
following updates, all batched along the target and mini-batch axes:
\begin{enumerate}
\item \emph{Bracket update.}
For each element, if $\cdfgmmn{n}(x) < u$ then the root lies above
$x$, so we set $\textsc{low} \leftarrow x$; otherwise we set
$\textsc{high} \leftarrow x$. This preserves the bracketing invariant.

\item \emph{Newton step.}
Using $F'(z) = \pdfgmmn{n}(z)$, the Newton update is
\begin{equation*}
    x_{\text{new}}
    \;=\; x \;-\;
    \frac{\cdfgmmn{n}(x) - u}{\pdfgmmn{n}(x)}.
\end{equation*}
The denominator is clamped from below by $10^{-8}$ to avoid division
by zero in regions where the mixture density is numerically negligible.

\item \emph{Safeguard.}
We accept the Newton step only if it falls inside the current bracket,
\(x_{\text{new}}\in[\textsc{low},\textsc{high}]\).
Otherwise (the Newton step overshot, or the local gradient was too
flat for Newton to be informative), we fall back to a bisection step
$x \leftarrow (\textsc{low}+\textsc{high})/2$ for that element.
\end{enumerate}

\paragraph{Convergence.}
Inside the bracket, Newton's method on a smooth, strictly increasing
$F$ converges quadratically, while the bisection fallback contracts
the bracket by a factor of two whenever Newton would have left it.
The number of iterations is fixed to 10, with no early exit (the same
budget for every element of the batch keeps the kernel launches static
and the whole loop compilable), and \Cref{app:icdf-validation} shows
that this budget is enough both for the downstream gradients and for
the reported njNLL\@. 

The gradients of $\cdfgmmninv{n}$ are obtained
\emph{analytically} via implicit differentiation
(\Cref{app:icdf-grad}), so we never backpropagate through the
iterations. The loop is wrapped in a \texttt{no\_grad} context, and
only the converged $z$ is fed into the closed-form expressions of
\Cref{prop:icdf-grad}.
Treating the solver as a black box that returns $z$, and pairing it
with the analytic gradients of \Cref{prop:icdf-grad}, buys three
things. The memory footprint of the backward pass is independent of the
number of Newton iterations, since none of the intermediate iterates are
retained on the autograd tape. The safeguarded update keeps the solver
from diverging in the tails, where the GMM density is small and a naive
Newton step would overshoot far outside the support of the marginals.
And because the bracket~\Cref{eq:icdf-bracket} is built directly from
the predicted GMM parameters, no dataset-dependent hyperparameters are
introduced.

\clearpage
\section{Derivation of the Analytic iCDF Gradients}\label{app:icdf-grad}
In this appendix we show how to compute the gradients of the marginal inverse CDF
$z_n = \cdfgmmninv{n}(u_n)$ with respect to the mixture parameters $(\pi_j, \mu_{j,n}, \sigma_{j,n}^2)$.
While \citet{Tewari2023.Estimation} already gave these gradients, we add a detailed derivation.
We do this because the inverse CDF has \emph{no closed form}, so at
training time $z_n$ is obtained numerically, via Newton's method.
Naively back-propagating through that solver would require unrolling all
of its iterations, which is expensive and numerically fragile.
The better alternative is \emph{implicit differentiation}.
It expresses the gradient of $z_n$ directly in terms of quantities we
already have at the converged solution, and never enters the solver.
The result is a handful of simple algebraic
expressions, given in \Cref{prop:icdf-grad} of
\Cref{sec:model}.
The derivation below shows where each of those expressions comes from.

\paragraph{What does the inverse CDF do?}
At a high level, $\cdfgmmninv{n}$ takes a number $u\in(0,1)$, interpreted as
a probability, and returns the value $z$ at which a Gaussian-mixture
random variable accumulates exactly that much probability mass below it.
There is no algebraic formula that solves this equation for a general
mixture, so we use Newton's method to find $z$ numerically.
\emph{But} once $z$ is found, the equation $\cdfgmmn{n}(z) = u$ is an
\emph{identity} (up to the solver tolerance) that one can differentiate.

\paragraph{Which partials do we need?}
The probability integral transform $u_n$ is treated as an input to the inverse CDF
and is upstream of the mixture parameters. Therefore, we do need a gradient with respect to it in order to train \model{}.
The mixture parameters $(\pi_j,\mu_{j,n},\sigma_{j,n}^2)$ are
predicted by the copula parameter functions $\phi_\pi$, $\phi_\mu$ and
$\phi_\sigma$ of \Cref{sec:model}, so the partial derivatives we
actually need are
\(
\partial z_n/\partial \pi_j,\;
\partial z_n/\partial \mu_{j,n},\;
\partial z_n/\partial \sigma_{j,n}^2.
\)
This restriction makes the derivation cleaner, since $u$ is
treated as a constant whenever we differentiate the defining identity.

Throughout the derivation $n$ is fixed and we drop the subscript $n$
wherever doing so does not cause confusion: write $z = z_n$,
$u = u_n$, $\mu_j = \mu_{j,n}$, $\sigma_j = \sigma_{j,n}$,
$\sigma_j^2 = \sigma_{j,n}^2$.
The marginal CDF and PDF are then:
\begin{equation}
    \cdfgmmn{n}(z)
    \;=\;
    \sum_{j=1}^{K} \pi_j\,\Phi\!\big(t_j(z)\big),
    \qquad
    t_j(z) \;:=\; \frac{z-\mu_j}{\sigma_j},
    \label{eq:app-cdf}
\end{equation}
\begin{equation}
    \pdfgmmn{n}(z)
    \;=\;
    \frac{d \cdfgmmn{n}}{d z}(z)
    \;=\;
    \sum_{j=1}^{K} \frac{\pi_j}{\sigma_j}\,
    \varphi \big(t_j(z)\big),
    \label{eq:app-pdf}
\end{equation}
where $\Phi$ and $\varphi$ are the standard-normal CDF and PDF\@.
The variable $t_j(z) = (z-\mu_j)/\sigma_j$ is the standardized
offset of $z$ from the mean of component $j$.

\paragraph{A note on the simplex constraint.}
The mixture weights satisfy $\pi_j > 0$ and $\sum_j \pi_j = 1$, but
\textbf{we do not enforce the simplex constraint inside this
derivation}. We treat each $\pi_j$ as a free positive scalar
in~\Cref{eq:app-cdf} and take partial derivatives with the other
$\pi_{j'}$ held fixed. The softmax (or whichever projection produces
$\pi$) handles the constraint outside.
This is what the code does. PyTorch's autograd composes the
softmax Jacobian with $\partial z/\partial \pi_j$ during back-propagation,
so working with unconstrained $\pi_j$ here gives the correct end-to-end
gradient once it is composed with the constraint-respecting layer.

\subsection{Setup: The Implicit Equation and Its Differential}

The forward pass solves
\begin{equation}
    \cdfgmmn{n}(z)\;=\;u
    \label{eq:app-implicit}
\end{equation}
for $z=z(u)$, which also depends on the mixture parameters.
\Cref{eq:app-implicit} is the \emph{defining identity}.
At the converged Newton iterate it holds as a true equality, up to the
residual left by the fixed iteration budget, which
\Cref{app:icdf-validation} shows to be negligible.
This is the only place in the derivation where we use that Newton's
method returned a true root, since~\Cref{eq:app-implicit} has to be an
identity before we can differentiate it.

The trick of implicit differentiation is to take the derivative of
\emph{both} sides of an identity and read off the gradient of the
implicit quantity (here $z$) without ever solving the equation symbolically.
Concretely, we differentiate~\Cref{eq:app-implicit} with respect to a
generic mixture parameter $\xi \in \{\pi_j, \mu_j, \sigma_j^2\}$.
Since $u$ does not depend on $\xi$, the right-hand side is zero:
\begin{equation}
    \frac{\partial}{\partial \xi}\,\cdfgmmn{n}\!\big(z(\xi)\big)
    \;=\;
    \underbrace{\frac{\partial \cdfgmmn{n}}{\partial z}(z)}_{\pdfgmmn{n}(z)}\,
    \frac{\partial z}{\partial \xi}
    \;+\;
    \frac{\partial \cdfgmmn{n}}{\partial \xi}\bigg|_{z\,\text{fixed}}
    \;=\; 0.
    \label{eq:app-master}
\end{equation}
The middle expression is the chain rule.
The variable $\xi$ enters $\cdfgmmn{n}(z)$ in two ways: through the
\emph{explicit} dependence of $\cdfgmmn{n}$ on its parameters
\emph{at a fixed argument $z$}, and through the \emph{implicit}
dependence of the argument $z$ itself on $\xi$ (because changing
$\xi$ changes the root of~\Cref{eq:app-implicit}).
The first term in~\Cref{eq:app-master} collects the implicit channel and
the second the explicit one.
Solving~\Cref{eq:app-master} for $\partial z/\partial\xi$ gives
\begin{equation}
    \boxed{\;
    \frac{\partial z}{\partial \xi}
    \;=\;
    -\,\frac{1}{\pdfgmmn{n}(z)}\,
    \frac{\partial \cdfgmmn{n}}{\partial \xi}\bigg|_{z\,\text{fixed}}\;}
    \label{eq:app-master-solved}
\end{equation}
Two remarks make this formula useful in practice.
First, the denominator $\pdfgmmn{n}(z)>0$ everywhere, because each Gaussian
component has full support on the real line, so the formula is
well-defined and never blows up in finite precision provided the
$\sigma_j$ are bounded away from zero (which is enforced by the
softplus parameterization in our network).
Second, every term on the right-hand side is something we already know
how to compute at the converged $z$, the value of the mixture PDF and
the explicit partial of the mixture CDF, both available in closed
form.
So this single line replaces the backward pass through the Newton
solver.

\subsection{Specializing for Each Parameter}

We now plug each of the three mixture parameters
into~\Cref{eq:app-master-solved}.
The pattern is always the same. Compute the explicit partial of
$\cdfgmmn{n}$ at fixed $z$, then divide by $-\pdfgmmn{n}(z)$.

\paragraph{(a) $\xi = \pi_j$ (gradient with respect to a mixture weight).}
Looking at~\Cref{eq:app-cdf}, only the $j$-th term in the sum
$\sum_{j'} \pi_{j'}\Phi(t_{j'}(z))$ contains $\pi_j$, and it is linear in
$\pi_j$.
The explicit partial of $\cdfgmmn{n}$ at fixed $z$ therefore reads off
immediately:
\begin{equation}
    \frac{\partial \cdfgmmn{n}}{\partial \pi_j}\bigg|_{z\,\text{fixed}}
    \;=\;
    \Phi\!\big(t_j(z)\big).
    \label{eq:app-Fpi}
\end{equation}
Substituting~\Cref{eq:app-Fpi} into~\Cref{eq:app-master-solved},
\begin{equation}
    \frac{\partial z}{\partial \pi_j}
    \;=\;
    -\,\frac{\Phi\!\big(t_j(z)\big)}{\pdfgmmn{n}(z)}.
    \label{eq:app-grad-pi}
\end{equation}
$\Phi(t_j(z)) \in (0,1)$ is the contribution of component $j$ to the
total probability mass at $z$, so increasing $\pi_j$ by a small amount
pushes more probability to the left of $z$.
To keep $\cdfgmmn{n}(z) = u$ unchanged we have to compensate by sliding
$z$ leftwards, hence the negative sign.
The magnitude of the response is set by how peaked the overall density is
near $z$, through the $1/\pdfgmmn{n}(z)$ factor. In flat regions a small
change in mass requires a large displacement of $z$, in peaked regions a
tiny one.

\paragraph{(b) $\xi = \mu_j$ (gradient with respect to a mean).}
From~\Cref{eq:app-cdf}, only the $j$-th term depends on $\mu_j$, and
$\partial t_j(z)/\partial \mu_j = -1/\sigma_j$ (with $z$ held fixed,
because $t_j$ depends on $\mu_j$ only through the numerator
$z-\mu_j$).
The chain rule, together with $\Phi'=\varphi$, gives
\begin{equation}
    \frac{\partial \cdfgmmn{n}}{\partial \mu_j}\bigg|_{z\,\text{fixed}}
    \;=\;
    \pi_j\,\varphi\!\big(t_j(z)\big)\,
    \frac{\partial t_j(z)}{\partial \mu_j}
    \;=\;
    -\,\frac{\pi_j}{\sigma_j}\,\varphi\!\big(t_j(z)\big).
    \label{eq:app-Fmu}
\end{equation}
Plugging~\Cref{eq:app-Fmu} into~\Cref{eq:app-master-solved} cancels
its minus sign and yields
\begin{equation}
    \frac{\partial z}{\partial \mu_j}
    \;=\;
    \frac{\pi_j\,\varphi\!\big(t_j(z)\big)/\sigma_j}{\pdfgmmn{n}(z)}.
    \label{eq:app-grad-mu}
\end{equation}
Intuitively, shifting the $j$-th component to the right (increasing
$\mu_j$) drags its probability mass to the right, so the cumulative
mass below any fixed $z$ \emph{decreases}.
To restore $\cdfgmmn{n}(z)=u$ we have to move $z$ rightwards as well,
matching the positive sign of the gradient.
The size of the response is governed by how much component $j$
contributes \emph{at $z$}, namely $\pi_j\varphi(t_j(z))/\sigma_j$,
divided by the total mixture density at $z$.
Components that are far from $z$ exert almost no influence, because
$\varphi(t_j(z))$ is exponentially small.

\paragraph{(c) $\xi = \sigma_j^2$ (gradient with respect to a variance).}
The more delicate piece is $\partial t_j(z)/\partial \sigma_j^2$.
Differentiating $t_j(z) = (z-\mu_j)\,\sigma_j^{-1}$ at fixed $z$ and
using the chain rule for $\sigma_j^{-1}$ as a function of $\sigma_j^2$,
\begin{equation}
    \frac{\partial t_j(z)}{\partial \sigma_j^2}
    \;=\;
    (z-\mu_j)\,\frac{d \sigma_j^{-1}}{d\sigma_j^2}
    \;=\;
    (z-\mu_j)\,\Big(-\tfrac12\,\sigma_j^{-3}\Big)
    \;=\;
    -\,\frac{z-\mu_j}{2\,\sigma_j^3}
    \;=\;
    -\,\frac{t_j(z)}{2\,\sigma_j^2},
    \label{eq:app-dtsigma2}
\end{equation}
where in the last step we used $(z-\mu_j)/\sigma_j = t_j(z)$ to fold the
prefactor back into $t_j$ and keep the formula compact.
Hence
\begin{align}
    \frac{\partial \cdfgmmn{n}}{\partial \sigma_j^2}\bigg|_{z\,\text{fixed}}
    &\;=\;
    \pi_j\,\varphi\!\big(t_j(z)\big)\,
    \frac{\partial t_j(z)}{\partial \sigma_j^2}\nonumber\\
    &\;=\;
    -\,\frac{\pi_j\,\varphi\!\big(t_j(z)\big)\,t_j(z)}{2\,\sigma_j^2}.
    \label{eq:app-Fsigma2}
\end{align}
Substituting~\Cref{eq:app-Fsigma2} into~\Cref{eq:app-master-solved}, the minus
sign once more cancels and we get
\begin{equation}
    \frac{\partial z}{\partial \sigma_j^2}
    \;=\;
    \frac{\pi_j\,\varphi\!\big(t_j(z)\big)\,t_j(z)}{2\,\sigma_j^2\,\pdfgmmn{n}(z)}
    \;=\;
    \frac{\pi_j\,\varphi\!\big(t_j(z)\big)\,(z-\mu_j)/\sigma_j}{2\,\sigma_j^2\,\pdfgmmn{n}(z)}.
    \label{eq:app-grad-sigma2}
\end{equation}

\clearpage
\section{Numerical Validation of the iCDF Solver and Its Gradients}\label{app:icdf-validation}

We validate the solver of \Cref{app:numerical-icdf} together with the
analytic gradients of \Cref{prop:icdf-grad}, derived in
\Cref{app:icdf-grad}.

\paragraph{Tails.}
The only safeguard on the input is a clip of $u$ to
$[\epsilon, 1-\epsilon]$, where $\epsilon$ is the machine epsilon of
the data type (\texttt{torch.finfo(u.dtype).eps}), i.e.\
$\epsilon = 1.19\times10^{-7}$ in float32.
Since the initial bracket~\Cref{eq:icdf-bracket} satisfies
$\cdfgmmn{n}(\textsc{low}) \le \Phi(-10)$ and
$\cdfgmmn{n}(\textsc{high}) \ge 1-\Phi(-10)$, and $\epsilon$ exceeds
$\Phi(-10)$ by many orders of magnitude, the clip also guarantees that
the bracket contains the root.
The clip was never active in our runs. Extreme values of $u$ are also
hard to produce. $u_n = \cdfmargn{n}(\ytgt_n \mid \xtgt_n, \Dctx)$ is
the CDF value of the frozen \margmodel{}, which is trained by maximum
likelihood, so $u_n < \epsilon$ would mean that \margmodel{} places
less than $\epsilon$ of its probability mass below an observed target,
which its own training objective penalizes heavily.
The same argument applies to $u_n > 1-\epsilon$ and the mass above the
target.

\paragraph{Gradient check.}
The analytic gradients of \Cref{prop:icdf-grad} agree with centered
finite differences of the solver output in float64.
\texttt{torch.autograd.gradcheck} passes for $K \in \{1, 3, 5\}$, with
maximum relative deviations of $8.2\times10^{-11}$ for the means,
$7.5\times10^{-11}$ for the variances, and $3.2\times10^{-9}$ for the
mixture weights.
Because the finite differences are taken through the solver itself, the
check covers both the closed forms and the assumption of
\Cref{app:icdf-grad} that the defining identity
$\cdfgmmn{n}(z_n) = u_n$ holds at the returned $z_n$.

\paragraph{Accuracy and effect on the reported numbers.}
To measure how much the remaining inversion error changes the njNLL, we
use GM-Cs with random parameters, namely means
$\mu_{j,n} \sim \mathcal{N}(0, 1)$, component scales
$\sigma_{j,n} \sim \mathcal{U}(0.5, 1.5)$, and random correlation
factors $V_j$.
For each mixture size $K$ and number of target points $N$ we draw four
parameter sets with 128 random queries $\mathbf{u} \in {(0,1)}^N$ each,
invert the $u$ values twice, and compute the njNLL from both
inversions.
The first inversion uses our solver. The second is a float64 reference
that runs 120 bisection steps on the bracket
$\big[\min_j q_{j,n}(u),\, \max_j q_{j,n}(u)\big]$, where
$q_{j,n}(u) = \mu_{j,n} + \sigma_{j,n}\Phi^{-1}(u)$ is the $u$-quantile
of component $j$.
Because $\cdfgmmn{n}$ is a convex combination of the component CDFs,
its $u$-quantile lies between the smallest and the largest component
quantile, so this bracket contains the root.
Each bisection step halves the bracket, so the bracket narrows to adjacent
float64 numbers after roughly $60$ steps; with $120$ steps the reference
is therefore accurate to float64 precision.
\Cref{tab:icdf-accuracy} reports the mean absolute difference between
the two njNLL values, which isolates the inversion error because the
marginal term is identical in both runs.

\begin{table}[ht]
    \centering
    \caption{Effect of the inversion error on the njNLL\@: mean absolute
    difference $|\Delta\,\mathrm{njNLL}|$ between
    our solver and a float64 bisection reference, on random GM-Cs with
    $K$ components and $N$ target points.}\label{tab:icdf-accuracy}
    \begin{tabular}{lcccc}
        \toprule
        & \multicolumn{4}{c}{Target points $N$} \\
        \cmidrule(lr){2-5}
        $K$ & $16$ & $32$ & $64$ & $128$ \\
        \midrule
        $1$  & $1.8\times10^{-7}$ & $3.9\times10^{-7}$ & $5.1\times10^{-7}$ & $5.6\times10^{-7}$ \\
        $3$  & $6.3\times10^{-6}$ & $1.1\times10^{-5}$ & $2.3\times10^{-5}$ & $2.7\times10^{-5}$ \\
        $5$  & $3.7\times10^{-6}$ & $1.6\times10^{-5}$ & $2.2\times10^{-5}$ & $3.0\times10^{-5}$ \\
        $7$  & $8.6\times10^{-6}$ & $2.3\times10^{-5}$ & $3.4\times10^{-5}$ & $3.8\times10^{-5}$ \\
        $10$ & $9.0\times10^{-6}$ & $1.9\times10^{-5}$ & $3.3\times10^{-5}$ & $4.7\times10^{-5}$ \\
        \bottomrule
    \end{tabular}
\end{table}

The difference grows with $N$ in every row and is more than an order
of magnitude smaller for $K=1$, where each marginal is a single
Gaussian, than for $K \ge 3$.
In every configuration it stays below $5\times10^{-5}$.
The grid covers the mixture and target-set sizes of our experiments.
$K$ ranges over the full search space of \Cref{tab:hp-gmcp}, and no
instance in our datasets has more than $85$ target points
(\Cref{tab:dataset-summary}).
The inversion error is therefore more than an order of magnitude below
the last digit reported in \Cref{tab:njnll}, and more than two orders
of magnitude below the smallest standard deviation of \model{}.
\FloatBarrier{}

\clearpage
\section{Computational Complexity}\label{app:complexity}

Evaluating the \model{} log-density at a target set of size $N$ is dominated by the $K$ multivariate Gaussian densities $\mathcal{N}(\mathbf{z};\mu_j,\Sigma_j)$. Each one requires a linear solve in $\Sigma_j$ and a $\log\det\Sigma_j$, which by Cholesky costs $\mathcal{O}(N^3)$ apiece and $\mathcal{O}(K N^3)$ in total. That worst case is never realized. By construction (\Cref{sec:model}), $\Sigma_j = D_j R_j D_j$ with $R_j = \tilde V_j \tilde V_j^\top + \Lambda_j$, where $\tilde V_j \in \R^{N \times H}$ stacks the normalized feature vectors $\tilde{\mathbf{v}}_{j,n}$, and $D_j$ and $\Lambda_j = \mathrm{diag}{\bigl(1 - {\lVert \tilde{\mathbf{v}}_{j,n} \rVert}^2\bigr)}_{n=1}^{N} \succ 0$ are diagonal. Each $R_j$ is therefore a rank-$H$ update of a positive diagonal matrix, so the Sherman-Morrison-Woodbury identity and the matrix determinant lemma give
\begin{align}
    R_j^{-1} &= \Lambda_j^{-1} - \Lambda_j^{-1} \tilde V_j {\bigl(I + \tilde V_j^\top \Lambda_j^{-1} \tilde V_j\bigr)}^{-1} \tilde V_j^\top \Lambda_j^{-1}, \\
    \log\det R_j &= \log\det\bigl(I + \tilde V_j^\top \Lambda_j^{-1} \tilde V_j\bigr) + \log\det \Lambda_j,
\end{align}
which need only $\mathcal{O}(N)$ diagonal work plus operations on the $H \times H$ matrix $I + \tilde V_j^\top \Lambda_j^{-1} \tilde V_j = I + V_j^\top V_j$. One component thus costs $\mathcal{O}(N H^2 + H^3)$ and the full log-density $\mathcal{O}\bigl(K (N H^2 + H^3)\bigr)$. For $H \ll N$ the cubic dependence on $N$ becomes linear. The surrounding stages are closed-form and negligible against this, at $\mathcal{O}(M D^2 + N D^2)$ for the encoder, $\mathcal{O}((N + 1) D^2)$ for the parameter MLPs and $\mathcal{O}(N L W)$ for the marginal flow with $L$ DSF blocks of width $W$. 
\paragraph{Cost of the numerical inversion.}
The inversion $z_n = \cdfgmmninv{n}(u_n)$ (\Cref{app:numerical-icdf}), with the implicit-function gradients of \Cref{app:icdf-grad}, is the one operation \model{} performs that a Gaussian-copula model does not. The solve is $\mathcal{O}(T K N)$ for $T$ Newton steps, and the analytic backward one further $\mathcal{O}(KN)$ pass, which \Cref{tab:icdf-flops} counts against the rest of the copula stage. The inversion stays under $2\%$ throughout, and its share \emph{falls} as $N$ grows. It is $\mathcal{O}(N)$, like the per-target parameter MLPs that dominate at $N=16$, while the $\mathcal{O}(K N^2 H)$ Gram construction and $\mathcal{O}(K N^3)$ factorization grow faster. The picture is unchanged at $K=10$ and under the structured evaluation above ($1.3\%$).

\begin{table}[htp!]
\centering
\footnotesize
\caption{Analytic FLOP counts (MFLOP) for one forward and backward evaluation of the \model{}
copula log-density ($K=5$, $H=32$, $D=64$, batch $64$, $T=10$). Ranges span charging a
transcendental at $1$ and at $10$ flops; only the inversion contains any. Counts are analytic:
matmul-based profilers report this elementwise solve as zero.}\label{tab:icdf-flops}
\setlength{\tabcolsep}{3mm}
\begin{tabular}{lccccc}
\toprule
\textbf{$N$} & \textbf{GMM inversion} & \textbf{Parameter MLPs} & \textbf{Gram} & \textbf{Cholesky} & \textbf{Share of total} \\
\midrule
16  & 1.0--2.0  & 118.6  & 5.2    & 0.4    & 0.78--1.58\% \\
64  & 3.9--8.0  & 474.5  & 83.9   & 28.0   & 0.66--1.33\% \\
128 & 7.8--15.9 & 949.0  & 335.5  & 223.7  & 0.51--1.04\% \\
256 & 15.6--31.9& 1897.9 & 1342.2 & 1789.6 & 0.31--0.62\% \\
\bottomrule
\end{tabular}
\end{table}

On an RTX 4060 the ten-iteration solve takes a flat $1.8\,\mathrm{ms}$ for all $N \le 256$, a \emph{single scalar} root-find still costs $1.76\,\mathrm{ms}$, and replaying the identical iterations as one captured CUDA graph recovers $79\%$ of it.

\paragraph{Empirical training time.}
When we compare epoch times, \model{} trains faster than TACTiS-2 on all four datasets, by up to $3.2\times$ (\Cref{tab:epoch-times}).

\begin{table}[htp!]
\centering
\footnotesize
\caption{Average wall-clock training time per epoch (seconds), on one NVIDIA A100, averaged over 5 splits. $^\dagger$Our adaptation of TACTiS-2 to sparse IMTS (\Cref{app:tactis-adapt}).}\label{tab:epoch-times}
\setlength{\tabcolsep}{3mm}
\begin{tabular}{lcccc}
\toprule
\textbf{Model} & \multicolumn{1}{c}{\texttt{USHCN}} & \multicolumn{1}{c}{\texttt{PhysioNet'12}} & \multicolumn{1}{c}{\texttt{MIMIC-III}} & \multicolumn{1}{c}{\texttt{MIMIC-IV}} \\
\midrule
\tactis{}        & 0.28 & 6.1 & 4.2 & 11.3 \\
\model{} (GM-C) & 0.22 & 2.2 & 4.0 & 3.5  \\
\bottomrule
\end{tabular}
\end{table}

\clearpage
\section{Synthetic Processes}\label{app:synthetic}

\Cref{sec:synthetic} uses three synthetic processes, the released pendulum (\Cref{app:pendulum}) and
the two L\'evy jump processes (\Cref{app:levy}). We specify them here. All three are built so that
their conditional copulas lie outside the class of mixtures of Gaussian copulas for \emph{any} number
of components. They violate the two conclusions of \Cref{prop:mix-gc-center} by different mechanisms.
The pendulum comes with ring-shaped bivariate marginals (\Cref{app:gmc-expressivity}). 
The L\'evy processes put mass on modes that a Mix-GC cannot move a component to. 

\paragraph{Models.}
All models share one \margmodel{} marginal model, which is trained first on the mNLL and then
frozen, so the models differ only in the copula. The \margmodel{} column of \Cref{tab:synthetic} uses
no copula, i.e.\ independent marginals. GCNP is a conditional Gaussian copula~\citep{Markou2021.Practical},
which we implement as a Mix-GC with a single component. Mix-GC and \model{} use $K = 10$ components,
and all three copulas use the same encoder architecture.

\paragraph{Training.}
Each training step draws a fresh batch of $64$ instances from the process, so every step sees
unseen data and there is no train set to overfit. We use Adam at a learning rate of
$3 \cdot 10^{-4}$ and select the checkpoint with the lowest loss on a fixed held-out batch of
$1{,}024$ instances, which also serves for early stopping with a patience of $4{,}000$ steps. Every
process gets the same budget of $20{,}000$ steps for the marginal model and $40{,}000$ for each
copula, set by the slowest of the three to fit.

\paragraph{Evaluation.}
Every model is trained $5$ times under different seeds, and every run is evaluated on the same
$10$ freshly drawn test sets of $1{,}000$ instances each, with $16$ targets per instance, so the
runs are paired across columns. \Cref{tab:synthetic} reports the njNLL as mean and standard
deviation across the $5$ training runs.

\subsection{The Released Pendulum}\label{app:pendulum}

The pendulum process is built so that (i) the context changes the dependence structure of the
targets, and not only their marginals, and (ii) its bivariate marginals lie outside the class of
mixtures of Gaussian copulas for any number of components (\Cref{prop:mix-gc-center}).

\paragraph{Generating process.}
Every instance is a single-channel series of a pendulum that is held at a height $A \sim \mathcal{U}[3, 4]$ and released from rest at an unobserved time $S$,
\begin{equation}\label{eq:pendulum}
    Y(t) \;=\;
    \begin{cases}
        A, & t \le S,\\
        A \cos\bigl(2\pi\,(t - S) / T(A)\bigr), & t > S,
    \end{cases}
    \qquad
    T(A) \;=\; T_0 \sqrt{A / \bar{A}} .
\end{equation}
At a fixed release angle the height is proportional to the arm length, and the period of a pendulum grows with the square root of its arm length. A single latent therefore sets both the amplitude and the period of the swing. We use $T_0 = 6$ and $\bar{A} = 3.5$, the midpoint of the height range. The release time is drawn as $S \sim \mathcal{U}[s_0, s_0 + T(A)]$ with $s_0 = 5$, i.e.\ uniformly over one period, so that the phase of the swing is uniform on $(0, 2\pi)$ at every height. All observations carry Gaussian noise with standard deviation $0.05$. The context consists of $6$ observations at uniformly drawn times in $[0, 4.5]$, all of which fall before the earliest possible release. The context therefore determines $A$, and with it $T(A)$, but carries no information about $S$. Each instance has $16$ targets at uniformly drawn times in $[4.5, 15]$, a window that contains the release.

\paragraph{The bivariate marginals are not Gaussian.}
Conditioned on the context, consider two targets $t_1 < t_2$ after the latest possible release, with lag $\Delta = t_2 - t_1$ and $\delta = 2\pi \Delta / T(A)$. With $\varphi \sim \mathcal{U}(0, 2\pi)$ the phase at $t_1$, the noise-free pair is $(y_1, y_2) = (A\cos\varphi,\, A\cos(\varphi + \delta))$, which lies on the ellipse
\begin{equation}\label{eq:pendulum-ellipse}
    y_1^2 - 2\, y_1 y_2 \cos\delta + y_2^2 \;=\; A^2 \sin^2\delta ,
\end{equation}
a circle if $\Delta = T(A)/4$. Each coordinate follows the arcsine law on $[-A, A]$, which is bimodal and symmetric around $0$, so the origin is mapped to the center $\mathbf{u}^\star = (\nicefrac{1}{2}, \nicefrac{1}{2})$ of the unit square. The ellipse keeps a distance of at least $A\sqrt{1 - \lvert\cos\delta\rvert}$ from the origin. Unless $\Delta$ is a multiple of $T(A)/2$, the copula density at $\mathbf{u}^\star$ is therefore close to $0$ by~\Cref{eq:center-ratio}, whereas every mixture of Gaussian copulas has density at least $1$ there by \Cref{prop:mix-gc-center}. The bivariate marginals of the pendulum are thus a conditional version of the \emph{circle} distribution of \Cref{fig:xo}. Since $\delta$ depends on $A$ through $T(A)$, their shape also changes with the context at a fixed lag, from an ellipse tilted one way through a circle to an ellipse tilted the other way (\Cref{fig:pendulum-bivariate}). A copula that reads the lag but not the context has to pick one of them and be wrong about the rest.

\paragraph{Where the copulas differ.}
The gap is widest on the bivariate marginals themselves, on pairs of targets drawn after the
latest release. \Cref{fig:pendulum-bivariate} shows why. GCNP and Mix-GC cannot place a hole at the
center, so they stay close to independence. The independent density is the product of the two
bimodal marginals and has four corner modes. \model{} recovers the ring and tilts it with the
context. Over the $5$ training runs of \Cref{tab:synthetic} the copula is worth $1.88$ njNLL
to \model{} but only $0.83$ to Mix-GC and $0.67$ to GCNP, so \model{} leads Mix-GC by
$1.05$. The anchored pairs these panels show, two targets after the latest release with the lag
drawn uniformly from $[0, 2.59]$, separate the models further. There, independent marginals reach an
njNLL of $1.850$, which GCNP improves by $0.15$ and Mix-GC by only $0.09$, while \model{} improves
it by $0.86$.

\begin{figure}[t]
    \centering
    \includegraphics[width=\textwidth]{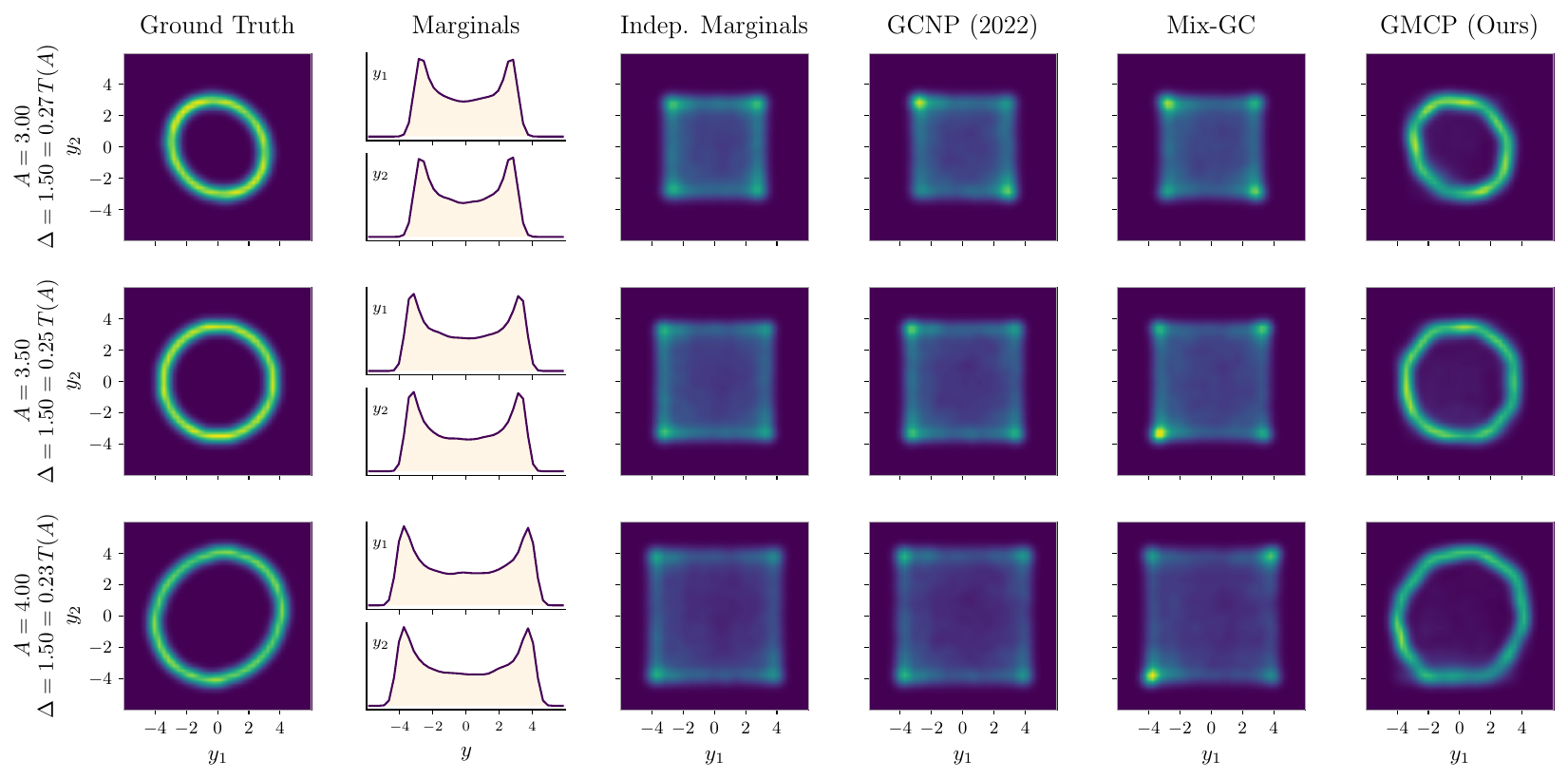}
    \caption{Bivariate marginals of the pendulum process for two targets after the release at a fixed lag $\Delta = 1.5$, conditioned on contexts with held heights $A = 3, 3.5, 4$ (rows, top to bottom). Columns show the ground-truth density, its univariate marginals, and smoothed histograms of samples from independent \margmodel{} marginals, GCNP, Mix-GC ($K = 10$), and \model{} ($K = 10$), all sharing the same marginals.}\label{fig:pendulum-bivariate}
\end{figure}

\subsection{L\'evy Jump Processes}\label{app:levy}

The pendulum violates \Cref{prop:mix-gc-center}~(i) by putting 0 density at the center of the unit
square. The L\'evy processes use the other side of the same proposition. They put mass on modes far
away from the center, which a Mix-GC cannot reach.

\paragraph{Generating process.}
Every instance is a single-channel \emph{pure-jump} series,
\begin{equation}\label{eq:levy}
    Y(t) \;=\; A\, L(t) ,
\end{equation}
with jump scale $A = 3$ and $L$ a L\'evy process with independent, stationary increments. We
evaluate two choices:
\begin{itemize}[leftmargin=1.5em,itemsep=2pt,topsep=2pt]
    \item \emph{compound Poisson} (CPP), $L(t) = N_\lambda(t)$, with unit jumps at rate
      $\lambda = 1$. It has finitely many jumps and the sharpest lattice of the three toys.
    \item \emph{Gamma}, $L(t) = G(t)$, a Gamma subordinator with variance rate $\nu = 0.3$. It has
      infinite activity, and its lattice is already partly smeared into a continuum.
\end{itemize}
Both are subordinators, so every jump is upward. The context consists of $12$ observations at
uniformly drawn times in $[0, 4]$ and each instance has $16$ targets at uniformly drawn times in
$(4, 7]$. Observations carry Gaussian noise with standard deviation $0.05$ on the compound Poisson
process, matching the pendulum, and $0.02$ on the Gamma one. Either way the noise is far below the
jump scale $A = 3$, so a jump is unambiguous in the context. The compound Poisson process needs the
wider noise because of its sharp lattice. 

\paragraph{What the copula has to express.}
The jump scale $A$ is fixed rather than drawn, so the hidden latent is the jump \emph{times} in the
target window, which fall strictly after the context closes. Anchor a pair at $a > 4$ and lag it by
$\Delta$. Conditional on the context, which fixes the level $Y(4)$, the noise-free pair sits on
\begin{equation}\label{eq:levy-pair}
    y_1 \;=\; Y(4) + A\, j , \qquad y_2 \;=\; y_1 + A\, k ,
\end{equation}
with $j$ and $k$ independent non-negative jump counts. The support is a lattice occupying a
quarter-plane, since $y_2 - y_1$ is bounded below by $0$ and unbounded above. Pushed through the
marginals, it becomes a grid of six to ten well-separated modes scattered across the unit square,
most of them far from the center. A Mix-GC can draw overlapping ellipses through the center but
cannot place a component on any of the modes. The free component means of a GM-C are the freedom to
put one component per mode.

\clearpage
\section{Experimental Details}\label{app:exp-details}

This appendix covers how the data are split and how the models are trained and tuned. We use the four benchmark datasets (USHCN, PhysioNet'12, MIMIC-III, MIMIC-IV) following the preprocessing protocols described by \citet{Yalavarthi2025.Reliable}, including the binning intervals, observation horizons, and channel selections reported there. We refer the reader to that work for the full dataset characteristics. Here, we focus on the protocol that governs how we train and evaluate our model and its baselines.

\subsection{Dataset Statistics}\label{app:dataset-stats}

For convenience we summarize the four datasets in \Cref{tab:dataset-summary}, including the per-instance target-count statistics (minimum, average, maximum). USHCN \citep{Menne2016.LongTerm} is a climatological dataset of daily measurements from 1{,}100 U.S.\ weather stations over four years (1996--2000), tracking five variables (snow precipitation, rain precipitation, snow depth, minimum temperature, maximum temperature). Following \citet{DeBrouwer2019.GRUODEBayes}, we transform USHCN into an IMTS by randomly dropping $95\%$ of the recorded observations. 
PhysioNet'12 \citep{Silva2012.Predicting} contains ICU records of $12{,}000$ patients monitored over a 48-hour period across 37 vital signs, binned into hourly intervals \citep{Che2018.Recurrent}. 
MIMIC-III \citep{Johnson2016.MIMICIII} provides ICU stays from Beth Israel Deaconess Medical Center, with 96 variables aggregated into 30-minute bins \citep{DeBrouwer2019.GRUODEBayes}.
MIMIC-IV \citep{Johnson2023.MIMICIV} records ICU stays from an academic medical center in Boston with 102 variables binned at the 1-minute resolution \citep{Bilos2021.Neural}. This fine-grained binning yields the largest number of observations and queries among the four datasets.

\begin{table}[ht]
    \centering
    \caption{Summary of dataset characteristics and preprocessing. The last three columns report the minimum, average, and maximum number of forecasting target points $N$ per instance.}\label{tab:dataset-summary}
    \begin{tabular}{lccccccc}
        \toprule
        & & & & & \multicolumn{3}{c}{Target points $N$} \\
        \cmidrule(lr){6-8}
        Dataset & Samples & Channels & Duration & Binning & $\min$ & $\bar{N}$ & $\max$ \\
        \midrule
        USHCN          & 1{,}100 (Stations)  & 5   & 4 Years   & 1 Week   &  3 & 3.3  & 6  \\
        PhysioNet'12   & 12{,}000 (Patients) & 37  & 48 Hours  & 1 Hour   &  1 & 19.8 & 53 \\
        MIMIC-III      & 21{,}000 (Patients) & 96  & 48 Hours  & 30 Mins  &  1 & 10.3 & 85 \\
        MIMIC-IV       & 18{,}000 (Patients) & 102 & 48 Hours  & 1 Min    &  1 & 7.8  & 79 \\
        \bottomrule
    \end{tabular}
\end{table}

\subsection{Data Splits}\label{app:splits}

We adopt the splitting protocol of \citet{Yalavarthi2025.Reliable} without modification, both to ensure a fair head-to-head comparison with previously reported numbers and to remove any degree of freedom that could be tuned in our favor. Each dataset is partitioned into training, validation, and test sets in a $70{:}20{:}10$ ratio. We repeat this partitioning under five distinct random seeds. We use the same seeds as experiments done in previous work~\citep{Yalavarthi2025.Reliable,Kloetergens2026.Probabilistic}. The test seed is tied to the fold index (e.g., Split~1 uses Seed~1), so the same five test sets are used by every model in our experiments. All reported numbers are means and standard deviations across these five folds.

\paragraph{Interpreting the standard deviation.} The five folds differ in which forecasting windows they expose to the model, and some folds contain inherently harder forecasting horizons than others (e.g.\ target points with a higher aleatoric uncertainty). Consequently, the per-fold scores of \emph{every} model on a given dataset shift up or down together as the fold becomes harder or easier, which inflates the across-fold standard deviation relative to the much smaller within-fold gap between models. The absolute magnitude of either model's standard deviation hence cannot rule out a significant difference. As a rule of thumb, when a model outperforms a baseline by about one reported standard deviation in this protocol, the improvement is consistent across folds and statistically meaningful.

\subsection{Hyperparameter Selection}\label{app:hyperparams}

For each model and dataset, we draw $10$ hyperparameter configurations from the search space below using random search, train each configuration on the first fold, and select the configuration with the best validation njNLL\@. The selected configuration is then retrained from scratch on all five folds; the test numbers we report come from this retrained configuration. 

\Cref{tab:hp-margflow,tab:hp-gmcp,tab:hp-tactis2} list the search spaces for \margmodel{}, \model{}, and our re-implemented TACTiS-2 baseline (see \Cref{app:tactis-adapt}). \model{} and TACTiS-2 share the same encoder, so the encoder-side hyperparameters (\texttt{copula\_n\_heads}, \texttt{copula\_hidden\_dim}) are drawn from identical grids; the remaining hyperparameters are specific to each copula module.

\begin{table}[h]
    \centering
    \caption{Hyperparameter search space for \margmodel{}}\label{tab:hp-margflow}
    \begin{tabular}{lp{0.5\linewidth}l}
        \toprule
        Hyperparameter & Description & Values \\
        \midrule
        \texttt{marg\_weight\_decay} & Weight decay of the marginal encoder       & $\{10^{-4}, 10^{-3}\}$ \\
        \texttt{marg\_n\_heads}      & Attention heads of the marginal encoder    & $\{1, 2, 4\}$ \\
        \texttt{marg\_hidden\_dim}   & Hidden dim of the marginal encoder         & $\{32, 64, 128\}$ \\
        \texttt{flow\_mlp\_layers}   & MLP layers of the conditioner network      & $\{1, 2, 3\}$ \\
        \texttt{flow\_mlp\_dim}      & MLP hidden dim of the conditioner network  & $\{32, 64, 128\}$ \\
        \texttt{flow\_layers}        & Number of normalizing-flow layers          & $\{1, 2, 3\}$ \\
        \texttt{flow\_hid\_dim}      & Hidden dim of each normalizing-flow layer  & $\{10, 20\}$ \\
        \bottomrule
    \end{tabular}
\end{table}

\begin{table}[h]
    \centering
    \caption{Hyperparameter search space for \model{} and variants.}\label{tab:hp-gmcp}
    \begin{tabular}{lcp{0.34\linewidth}l}
        \toprule
        Hyperparameter & Symbol & Description & Values \\
        \midrule
        \texttt{copula\_components}    & $K$ & Number of mixture components             & $\{1, 3, 5, 7, 10\}$ \\
        \texttt{copula\_n\_heads}      & --- & Attention heads of the encoder           & $\{1, 2, 4\}$ \\
        \texttt{copula\_hidden\_dim}   & $D$ & Hidden dim of the encoder                & $\{32, 64, 128\}$ \\
        \texttt{corr\_net\_hidden\_dim}& $H$ & Hidden dim of the MLPs producing $\mu$, $\sigma$, and the Gram-matrix vectors & $\{16, 32, 64, 128\}$ \\
        \bottomrule
    \end{tabular}
\end{table}

\begin{table}[h]
    \centering
    \caption{Hyperparameter search space for our re-implemented TACTiS-2 baseline. The copula-side ranges match those reported by \citet{Ashok2023.TACTiS2}; the encoder-side ranges (\texttt{copula\_hidden\_dim}, \texttt{copula\_n\_heads}) match \model{}.}\label{tab:hp-tactis2}
    \begin{tabular}{lll}
        \toprule
        Hyperparameter & Description & Values \\
        \midrule
        \texttt{copula\_attn\_heads}   & Attention heads of the attentional copula  & $\{2, 4, 8\}$ \\
        \texttt{copula\_attn\_layers}  & Attention layers of the attentional copula & $\{1, 2, 3, 4\}$ \\
        \texttt{copula\_attn\_dim}     & Attention dim of the attentional copula    & $\{8, 16, 32\}$ \\
        \texttt{copula\_mlp\_layers}   & MLP layers of the attentional copula       & $\{1, 2, 3\}$ \\
        \texttt{copula\_mlp\_dim}      & MLP hidden dim of the attentional copula   & $\{32, 64, 128\}$ \\
        \texttt{copula\_resolution}    & CDF discretization resolution              & $\{8, 16, 32\}$ \\
        \texttt{copula\_dropout}       & Dropout rate of the attentional copula     & $\{0.0, 0.1, 0.2\}$ \\
        \texttt{copula\_hidden\_dim}   & Hidden dim of the encoder                  & $\{32, 64, 128\}$ \\
        \texttt{copula\_n\_heads}      & Attention heads of the encoder             & $\{1, 2, 4\}$ \\
        \bottomrule
    \end{tabular}
\end{table}

\subsection{Training Setup}\label{app:training}

Unless stated otherwise, we train \model{} in the \emph{two-stage} scheme of \citet{Ashok2023.TACTiS2}. We first fit the marginal model \margmodel{} on its own, by treating all target points as independent, and then train the copula on top of the frozen marginals. As discussed in \Cref{sec:model}, this is an optimization choice rather than a property of \model{}\@. \model{}-Joint, reported in \Cref{sec:experiments}, trains the same architecture jointly under njNLL\@.

All models are trained with the AdamW optimizer \citep{Kingma2017.Adam,Loshchilov2019.Decoupled} using an initial learning rate of $10^{-3}$, weight decay of $10^{-3}$, and a batch size of $64$. We schedule the learning rate with reduce-on-plateau (factor $0.5$, patience of $5$ validation epochs without improvement) and train for at most $2000$ epochs with early stopping (patience of $30$ epochs on the validation njNLL). All experiments use float32 precision. For the MIMIC-IV dataset we reduce the batch size to $32$ when memory pressure requires it; we apply this reduction uniformly across all models on that dataset. We train on a single NVIDIA A100 GPU per fold.

\subsection{Evaluation Metrics}\label{app:metrics}

We evaluate every model on the held-out test split of each fold using the normalized joint negative log-likelihood (njNLL) and the marginal negative log-likelihood (mNLL). The njNLL is defined in~\Cref{eq:njnll}; the mNLL evaluates the predictive density at each target coordinate independently and averages over target points,
\begin{equation}
    \mathrm{mNLL}(\Ytgt \mid \Xtgt, \Dctx)
    \;=\;
    \frac{1}{N}\sum_{n=1}^{N} -\log p\bigl(\ytgt_n \,\big|\, \xtgt_n, \Dctx\bigr).
    \label{eq:mnll}
\end{equation}
For a factorized model such as \margmodel{}, the njNLL reduces to an average of marginal negative log-likelihoods, yet its njNLL and mNLL differ in absolute value. The difference is due only to how the two metrics are aggregated over the test set. The njNLL first averages over the $N$ target points of each instance and then over instances, so every instance carries equal weight. The mNLL averages once, over all target points of the test set pooled together, so instances with more target points carry more weight.

\subsection{Univariate Wasserstein-1 Distance}\label{app:wd}

The Wasserstein distances reported in \Cref{tab:wd} compare two \emph{sample sets} of the same model,
the samples of \margmodel{} alone and the samples of the full \margmodel{}+copula setup, either \model{}
or TACTiS-2. The metric asks whether attaching the copula leaves
the univariate marginals where \margmodel{} put them, which is what the definition of a copula requires.
\model{} guarantees it by construction (\Cref{app:mc-proof}) and TACTiS-2's attentional copula does not.
The Control row quantifies the irreducible sampling error. It is computed by drawing two independent sets
of $S=1000$ samples from the same \margmodel{} under different random seeds and applying the same
estimator to those two sample sets.
For two probability measures $P$ and $Q$ on $\R$ with cumulative distribution functions $F_P$ and $F_Q$, the Wasserstein-1 distance admits the closed form
\begin{equation}
W_1(P, Q) \;=\; \int_{\R} \bigl|F_P(t) - F_Q(t)\bigr|\, dt.
\end{equation}
Given two equal-sized samples ${\{a_i\}}_{i=1}^{S}$ and ${\{b_i\}}_{i=1}^{S}$ drawn from $P$ and $Q$, the empirical Wasserstein-1 distance reduces to the average absolute gap between order statistics,
\begin{equation}
\widehat{W}_1\bigl(\{a_i\}, \{b_i\}\bigr) \;=\; \frac{1}{S} \sum_{i=1}^{S} \bigl|a_{(i)} - b_{(i)}\bigr|,
\end{equation}
where $a_{(1)} \le \dots \le a_{(S)}$ and $b_{(1)} \le \dots \le b_{(S)}$ denote the sorted samples.
At every target point of a test fold we draw $S = 1000$ samples from each of the two setups and pool
them, across target points, forecasting dimensions and test instances, into two empirical distributions of
equal size, to which the estimator above is applied. This yields one distance per fold, and we report its
mean and standard deviation over the five splits. Note that pooling before comparing measures the marginal
distribution of a test fold as a whole rather than that of an individual target point; \Cref{app:tactis-breaker}
examines marginalization consistency in the sharper, per-target-set form.

\subsection{Adapting TACTiS-2 to Sparse IMTS}\label{app:tactis-adapt}

The attentional copula is introduced in TACTiS~\citep{Drouin2022.TACTiS} and reused in TACTiS-2~\citep{Ashok2023.TACTiS2}. The two methods share the same attentional copula module and differ in the training scheme. In the original TACTiS, a single shared encoder feeds both the marginals and the copula, and all components are trained jointly under the joint negative log-likelihood. In TACTiS-2, the marginals and the copula are produced by two separate encoders and trained in two stages. The marginals are trained first and then frozen, and the copula is trained on top of them. The two methods also differ in how the autoregressive copula orders the target points: TACTiS draws a random permutation of the targets at every training step, whereas TACTiS-2 does not permute them. \citet{Ashok2023.TACTiS2} report that this two-stage protocol outperforms the joint TACTiS training on all of their tasks. It is also the protocol used by \model{}, so comparing against the TACTiS-2 form of the attentional copula gives the most direct copula-vs-copula comparison.

One component of the original TACTiS-2 does not transfer to our setting unchanged, namely its encoder. It was designed for moderately irregular but still relatively dense multivariate time series, in which every channel is observed at most timestamps. The IMTS benchmarks we use are far sparser, with channels that may be observed only a handful of times over the whole observation window, and the original TACTiS-2 encoder underperforms badly in this regime.

To isolate the contribution of the copula module itself, we therefore re-implement the attentional copula on top of the same encoder that \model{} uses (described in \Cref{sec:imts}). Concretely, our TACTiS-2 baseline keeps the original attentional copula, the original two-stage training protocol and the unpermuted target ordering of TACTiS-2, and only swaps the encoder for one that is suited to sparse IMTS\@. This is also the configuration in which it performs best on our benchmarks, so we use it everywhere we refer to TACTiS-2 in the main paper.

\clearpage
\section{Complete Likelihood Results}\label{app:full-results}

\Cref{tab:njnll,tab:mnll} in the main paper restrict the comparison to the models that are within reach
of \model{} and \margmodel{}, namely the jointly trained IMTS forecasters and, for njNLL, the copula
baselines. We report additional baselines here. These are the univariate
neural-ODE models GRU-ODE-Bayes~\citep{DeBrouwer2019.GRUODEBayes}, NeuralFlows~\citep{Bilos2021.Neural}
and CRU~\citep{Schirmer2022.Modeling}, together with multitask Gaussian Process Regression
(GPR)~\citep{Durichen2015.Multitask}. They were moved out of the main paper for space reasons.
As \Cref{tab:njnll-full,tab:mnll-full} show, none of them is competitive on any dataset under either
metric, and all of them are separated from the models in the main tables by a margin far larger than
the spread across folds. The three neural-ODE models emit an independent Gaussian per target point, so
their joint score carries no information about dependence beyond what their marginal score already
says. GPR is consistent, but restricted to multivariate Gaussians. Where a baseline number is available from previous
work under the identical protocol, we report it
directly~\citep{Yalavarthi2025.Reliable,Kloetergens2026.Probabilistic}. CRU runs out of memory on
\texttt{MIMIC-IV}, which we mark as OOM\@.
\Cref{tab:njnll-full} additionally reports \model{} with $K=1$, to which we refer as Gaussian Copula Neural Process (GCNP)~\citep{Markou2021.Practical}. It improves on independent marginals everywhere, but trails Mix-GC and \model{} with $K>1$ on all
four datasets. Its Energy Score is in
\Cref{tab:energy-score}; \Cref{tab:mse,tab:crps} omit it because a copula leaves the \margmodel{}
marginals untouched (\Cref{app:mc-proof}).

\begin{table}[!htp]
\centering
\caption{Complete version of \Cref{tab:njnll}: \textbf{njNLL} (Normalized Joint Negative Log-Likelihood)
across datasets, including the univariate and Gaussian-process baselines and the GCNP copula that
\Cref{tab:njnll} omits.
Lower values indicate better performance. We present the mean and standard deviation of 5 runs.
The best model is marked in \textbf{bold}. $^\dagger$Our adaptation of TACTiS-2 to sparse IMTS (\Cref{app:tactis-adapt}).}\label{tab:njnll-full}
\small
\setlength{\tabcolsep}{1.5mm}
\begin{tabular}{llrrrr}
\toprule
& \textbf{Model} & \multicolumn{1}{c}{\texttt{USHCN}} & \multicolumn{1}{c}{\texttt{PhysioNet'12}} & \multicolumn{1}{c}{\texttt{MIMIC-III}} & \multicolumn{1}{c}{\texttt{MIMIC-IV}} \\
\midrule
\multirow{4}{*}{\rotatebox{90}{Univariate}}
& GRU-ODE-Bayes & \metric{0.766}{0.159} & \metric{0.501}{0.001} & \metric{0.961}{0.064} & \metric{0.823}{0.318} \\
& NeuralFlows & \metric{0.775}{0.152} & \metric{0.496}{0.000} & \metric{0.998}{0.111} & \metric{0.689}{0.087} \\
& CRU & \metric{0.761}{0.191} & \metric{1.057}{0.007} & \metric{1.234}{0.076} & \metricoom{} \\
& \margmodel{} & \metric{-3.969}{0.305} & \metric{-0.442}{0.021} & \metric{-0.505}{0.208} & \metric{-1.819}{0.065} \\
\midrule
\multirow{5}{*}{\rotatebox{90}{Joint}}
& ProFITi & \metric{-3.226}{0.225} & \metric{-0.647}{0.078} & \metric{-0.377}{0.032} & \metric{-1.777}{0.066} \\
& GPR & \metric{2.011}{1.376} & \metric{1.367}{0.074} & \metric{3.146}{0.359} & \metric{2.789}{0.057} \\
& MOSES & \metric{-3.357}{0.176} & \metric{-0.491}{0.041} & \metric{-0.305}{0.027} & \metric{-1.668}{0.097} \\
& CircuITS & \metric{-3.789}{0.218} & \metric{-0.550}{0.013} & \metric{-0.574}{0.080} & \metric{-2.113}{0.044} \\
& \model{}-Joint & \metric{-3.907}{0.269} & \metric{-0.489}{0.399} & \metric{-0.508}{0.309} & \metric{-2.019}{0.049} \\
\midrule
\multirow{4}{*}{\rotatebox{90}{Copula}}
& \tactis{} & \metricb{-4.254}{0.294} & \metric{-0.732}{0.115} & \metric{-0.780}{0.282} & \metric{-2.053}{0.170} \\
& GCNP & \metric{-3.998}{0.325} & \metric{-0.680}{0.019} & \metric{-0.642}{0.210} & \metric{-1.873}{0.067} \\
& Mix-GC & \metric{-4.078}{0.308} & \metric{-0.728}{0.018} & \metric{-0.779}{0.216} & \metric{-2.080}{0.064} \\
& \textbf{\model{}} & \metric{-4.135}{0.304} & \metricb{-0.745}{0.019} & \metricb{-0.835}{0.218} & \metricb{-2.201}{0.058} \\
\bottomrule
\end{tabular}%
\end{table}

\begin{table}[!htp]
\centering
\caption{Complete version of \Cref{tab:mnll}: \textbf{mNLL} (Marginal Negative Log-Likelihood) across
datasets, including the univariate and Gaussian-process baselines that \Cref{tab:mnll} omits.
Lower values indicate better performance. We present the mean and standard deviation of 5 runs. The best model is marked in \textbf{bold}.}\label{tab:mnll-full}
\small
\setlength{\tabcolsep}{1.5mm}
\begin{tabular}{llrrrr}
\toprule
& \textbf{Model} & \multicolumn{1}{c}{\texttt{USHCN}} & \multicolumn{1}{c}{\texttt{PhysioNet'12}} & \multicolumn{1}{c}{\texttt{MIMIC-III}} & \multicolumn{1}{c}{\texttt{MIMIC-IV}} \\
\midrule
\multirow{5}{*}{\rotatebox{90}{Joint}}
& ProFITi & \metric{-3.324}{0.206} & \metric{-0.016}{0.085} & \metric{0.408}{0.030} & \metric{0.500}{0.322} \\
& GPR & \metric{1.235}{0.096} & \metric{1.161}{0.065} & \metric{1.341}{0.009} & \metric{1.161}{0.010} \\
& MOSES & \metric{-3.355}{0.156} & \metric{-0.271}{0.028} & \metric{0.163}{0.026} & \metric{-0.634}{0.017} \\
& CircuITS & \metric{-3.717}{0.201} & \metric{-0.287}{0.037} & \metric{0.095}{0.107} & \metric{-0.731}{0.063} \\
& \model{}-Joint & \metric{-3.663}{0.210} & \metric{-0.176}{0.470} & \metric{0.108}{0.238} & \metric{-0.677}{0.044} \\
\midrule
\multirow{4}{*}{\rotatebox{90}{Univariate}}
& GRU-ODE-Bayes & \metric{0.776}{0.172} & \metric{0.504}{0.061} & \metric{0.839}{0.030} & \metric{0.876}{0.589} \\
& NeuralFlows & \metric{0.775}{0.180} & \metric{0.492}{0.029} & \metric{0.866}{0.097} & \metric{1.796}{0.050} \\
& CRU & \metric{0.762}{0.180} & \metric{0.931}{0.019} & \metric{1.209}{0.044} & \metricoom{} \\
& \textbf{\margmodel{}} & \metricb{-3.948}{0.294} & \metricb{-0.479}{0.021} & \metricb{-0.158}{0.205} & \metricb{-0.833}{0.077} \\
\bottomrule
\end{tabular}
\end{table}

\clearpage
\section{Sampling Results}\label{app:sampling-results}

We complement the likelihood-based evaluation in the main paper with three sample-based metrics, namely Mean Squared Error (MSE) for point forecasts (\Cref{tab:mse}), Energy Score (ES) for multivariate distributional accuracy (\Cref{tab:energy-score}) and Continuous Ranked Probability Score (CRPS) for marginal distributional accuracy (\Cref{tab:crps}). For each test instance we draw $S = 1000$ samples $\hat{\mathbf{y}}^{(s)} \in \R^N$ from the predictive distribution conditioned on $(\Xtgt, \Dctx)$, and compare them against the ground-truth target outputs $\Ytgt \in \R^N$. All three metrics are reported as means across the test set and across the five folds.
\paragraph{MSE.}
We use the sample mean $\bar{\hat{y}}_n = \tfrac{1}{S}\sum_{s=1}^S \hat{y}^{(s)}_n$ as the point forecast and report the per-coordinate squared error averaged over target points,
\begin{equation}
    \mathrm{MSE}(\hat{\mathbf{y}}, \mathbf{y})
    \;=\;
    \frac{1}{N}\sum_{n=1}^{N} {(\bar{\hat{y}}_n - y_n)}^2.
    \label{eq:mse}
\end{equation}

\paragraph{Energy Score.}
The Energy Score~\citep{Gneiting2007.Strictly} is a strictly proper scoring rule for multivariate distributions. For a predictive distribution $\hat{F}$ and target $\mathbf{y} \in \R^N$ it is defined as
\begin{equation}
    \mathrm{ES}(\hat{F}, \mathbf{y})
    \;=\;
    \mathbb{E}_{\mathbf{Y}\sim\hat{F}}\,\|\mathbf{Y} - \mathbf{y}\|_2
    \;-\;
    \tfrac{1}{2}\,\mathbb{E}_{\mathbf{Y},\mathbf{Y}'\sim\hat{F}}\,\|\mathbf{Y} - \mathbf{Y}'\|_2,
\end{equation}
where $\mathbf{Y}, \mathbf{Y}'$ are independent draws from $\hat{F}$. Given $S$ samples, we use the standard unbiased estimator
\begin{equation}
    \widehat{\mathrm{ES}}
    \;=\;
    \frac{1}{S}\sum_{s=1}^{S} \|\hat{\mathbf{y}}^{(s)} - \mathbf{y}\|_2
    \;-\;
    \frac{1}{2\,S\,(S-1)} \sum_{s \neq s'} \|\hat{\mathbf{y}}^{(s)} - \hat{\mathbf{y}}^{(s')}\|_2.
    \label{eq:es}
\end{equation}

\paragraph{CRPS.}
The Continuous Ranked Probability Score~\citep{Gneiting2007.Strictly} is the univariate counterpart of the Energy Score. For a one-dimensional predictive CDF $\hat{F}_n$ and target $y_n \in \R$,
\begin{equation}
    \mathrm{CRPS}(\hat{F}_n, y_n)
    \;=\;
    \int_{\R} {\bigl(\hat{F}_n(z) - \ind\{z \geq y_n\}\bigr)}^2 \, \mathrm{d}z
    \;=\;
    \mathbb{E}\,|Y_n - y_n|
    \;-\;
    \tfrac{1}{2}\,\mathbb{E}\,|Y_n - Y'_n|,
    \label{eq:crps-def}
\end{equation}
with $Y_n, Y'_n \sim \hat{F}_n$ independent. We report the average across target points, estimated from the per-coordinate samples $\hat{y}^{(s)}_n$ as
\begin{equation}
    \widehat{\mathrm{CRPS}}
    \;=\;
    \frac{1}{N}\sum_{n=1}^{N}
    \left[
        \frac{1}{S}\sum_{s=1}^{S} |\hat{y}^{(s)}_n - y_n|
        \;-\;
        \frac{1}{2\,S\,(S-1)} \sum_{s \neq s'} |\hat{y}^{(s)}_n - \hat{y}^{(s')}_n|
    \right].
    \label{eq:crps}
\end{equation}
Lower values are better for all three metrics.

\begin{table}[!htp]
\centering
\caption{Comparing models w.r.t.\ \textbf{MSE} (Mean Squared Error) across datasets. Lower values indicate better performance. We present the mean and standard deviation of 5 runs. The best model is marked in \textbf{bold}, the second best is \underline{underlined}. Rows marked with $^*$ use a \margmodel{} that is optimized for CRPS instead of likelihood.}\label{tab:mse}
\small
\setlength{\tabcolsep}{2mm}
\begin{tabular}{lrrrr}
\toprule
\textbf{Model} & \multicolumn{1}{c}{\texttt{USHCN}} & \multicolumn{1}{c}{\texttt{PhysioNet'12}} & \multicolumn{1}{c}{\texttt{MIMIC-III}} & \multicolumn{1}{c}{\texttt{MIMIC-IV}} \\
\midrule
GRU-ODE & \metric{0.410}{0.106} & \metric{0.329}{0.004} & \metric{0.479}{0.044} & \metric{0.365}{0.012} \\
NeuralFlows & \metric{0.424}{0.110} & \metric{0.331}{0.006} & \metric{0.479}{0.045} & \metric{0.374}{0.017} \\
CRU & \metricb{0.290}{0.060} & \metric{0.475}{0.015} & \metric{0.725}{0.037} & \metricoom{} \\
ProFITi & \metric{0.308}{0.061} & \metric{0.305}{0.007} & \metric{0.548}{0.063} & \metric{0.389}{0.015} \\
GPR & \metric{0.597}{0.110} & \metric{0.575}{0.059} & \metric{0.862}{0.016} & \metric{0.609}{0.014} \\
MOSES & \metric{0.411}{0.099} & \metric{0.307}{0.006} & \metric{0.517}{0.057} & \metric{0.342}{0.028} \\
CircuITS & \metricu{0.306}{0.030} & \metricu{0.292}{0.000} & \metricu{0.475}{0.050} & \metric{0.285}{0.001} \\
\midrule
\margmodel{} & \metric{0.389}{0.154} & \metric{0.310}{0.016} & \metric{0.480}{0.047} & \metricu{0.278}{0.003} \\
\margmodel{}$^*$ & \metric{0.316}{0.082} & \metricb{0.290}{0.001} & \metricb{0.425}{0.031} & \metricb{0.251}{0.001} \\
\bottomrule
\end{tabular}
\end{table}

\begin{table}[!htp]
\centering
\caption{Comparing models w.r.t.\ \textbf{Energy Score} (multivariate) across datasets. Lower values indicate better performance. We present the mean and standard deviation of 5 runs. The best model is marked in \textbf{bold}, the second best is \underline{underlined}. Rows marked with $^*$ use a \margmodel{} that is optimized for CRPS instead of likelihood. GCNP is the single-component ($K = 1$) instance of \model{}, i.e.\ a conditional Gaussian copula~\citep{Markou2021.Practical}; it shares the frozen \margmodel{} marginals of the row above it. $^\dagger$Our adaptation of TACTiS-2 to sparse IMTS (\Cref{app:tactis-adapt}).}\label{tab:energy-score}
\small
\setlength{\tabcolsep}{2mm}
\begin{tabular}{lrrrr}
\toprule
\textbf{Model} & \multicolumn{1}{c}{\texttt{USHCN}} & \multicolumn{1}{c}{\texttt{PhysioNet'12}} & \multicolumn{1}{c}{\texttt{MIMIC-III}} & \multicolumn{1}{c}{\texttt{MIMIC-IV}} \\
\midrule
NeuralFlows & \metric{0.661}{0.059} & \metric{1.691}{0.001} & \metric{1.381}{0.033} & \metric{0.982}{0.009} \\
ProFITi & \metricb{0.452}{0.044} & \metricb{0.879}{0.303} & \metric{1.606}{0.168} & \metric{0.808}{0.003} \\
MOSES & \metric{0.552}{0.044} & \metric{1.599}{0.013} & \metric{1.353}{0.033} & \metric{0.906}{0.029} \\
CircuITS & \metricu{0.468}{0.020} & \metric{1.613}{0.000} & \metric{1.489}{0.040} & \metric{0.927}{0.050} \\
\tactis{} & \metric{0.477}{0.052} & \metric{1.613}{0.018} & \metric{1.305}{0.031} & \metric{0.825}{0.007} \\
\midrule
\margmodel{} & \metric{0.483}{0.054} & \metric{1.606}{0.012} & \metric{1.349}{0.033} & \metric{0.826}{0.005} \\
GCNP & \metric{0.482}{0.055} & \metric{1.594}{0.011} & \metric{1.341}{0.033} & \metric{0.824}{0.005} \\
\model{} & \metric{0.482}{0.054} & \metric{1.593}{0.010} & \metric{1.340}{0.033} & \metric{0.826}{0.005} \\
\margmodel{}$^*$ & \metric{0.472}{0.043} & \metric{1.585}{0.002} & \metricu{1.277}{0.030} & \metricu{0.789}{0.002} \\
\model{}$^*$ & \metric{0.472}{0.043} & \metricu{1.576}{0.002} & \metricb{1.270}{0.029} & \metricb{0.788}{0.002} \\
\bottomrule
\end{tabular}
\end{table}

\begin{table}[!htp]
\centering
\caption{Comparing models w.r.t.\ \textbf{CRPS} (Continuous Ranked Probability Score) on marginals across datasets. Lower values indicate better performance. We present the mean and standard deviation of 5 runs. The best model is marked in \textbf{bold}, the second best is \underline{underlined}. Rows marked with $^*$ use a \margmodel{} that is optimized for CRPS instead of likelihood.}\label{tab:crps}
\small
\setlength{\tabcolsep}{2mm}
\begin{tabular}{lrrrr}
\toprule
\textbf{Model} & \multicolumn{1}{c}{\texttt{USHCN}} & \multicolumn{1}{c}{\texttt{PhysioNet'12}} & \multicolumn{1}{c}{\texttt{MIMIC-III}} & \multicolumn{1}{c}{\texttt{MIMIC-IV}} \\
\midrule
NeuralFlows & \metric{0.306}{0.028} & \metric{0.277}{0.003} & \metric{0.308}{0.004} & \metric{0.281}{0.004} \\
ProFITi & \metricb{0.182}{0.007} & \metric{0.271}{0.003} & \metric{0.319}{0.003} & \metric{0.279}{0.012} \\
MOSES & \metric{0.220}{0.019} & \metric{0.260}{0.002} & \metric{0.296}{0.005} & \metric{0.245}{0.010} \\
CircuITS & \metricb{0.182}{0.010} & \metric{0.252}{0.001} & \metricu{0.286}{0.010} & \metric{0.221}{0.001} \\
\midrule
\margmodel{} & \metricu{0.187}{0.022} & \metricu{0.251}{0.001} & \metric{0.289}{0.005} & \metricu{0.217}{0.001} \\
\margmodel{}$^*$ & \metric{0.189}{0.016} & \metricb{0.250}{0.000} & \metricb{0.271}{0.004} & \metricb{0.208}{0.001} \\
\bottomrule
\end{tabular}
\end{table}

\paragraph{\texorpdfstring{What the Energy Score Does and Does Not Resolve}{What the Energy Score Does and Does Not Resolve}}\label{app:es-discrimination}

The Energy Score is strictly proper, meaning its expected value under the data-generating distribution is minimized
by that exact distribution, and by no other (for distributions with finite first moment;
\citealp{Gneiting2007.Strictly}). The true predictive distribution is therefore its unique minimizer. Strict
propriety is a statement about that minimizer alone. It says nothing about how steeply the score rises as
a model is perturbed away from the truth, and nothing about whether it rises at the same rate for a
perturbation of the marginals as for one of the dependence structure. Those rates decide whether a
difference in ES between two models is readable at all, and they have been measured.
\citet{Marcotte2023.Regions} map out the \emph{regions of
reliability} of ES-based comparisons as a function of dimension and of the number of samples used in the
estimator~\Cref{eq:es}, and find the reliable region to be considerably smaller than common practice
assumes. Our forecasting targets reach $N = 85$ points per instance (\Cref{tab:dataset-summary}), so this
is the regime we are evaluating in.

\begin{table}[!htp]
\centering
\caption{What the Energy Score resolves, decomposed from \Cref{tab:energy-score,tab:njnll}. The
\emph{copula} column is the improvement obtained by attaching \model{} to \margmodel{}, which alters the
dependence structure and provably leaves the marginals untouched; the \emph{marginals} column is the
improvement obtained by fitting \margmodel{} on CRPS instead of likelihood (\Cref{app:crps-training}),
which alters the marginals and leaves the copula stage untouched. Positive values are improvements. The
copula is worth $0.17$--$0.38$ njNLL and next to nothing in ES, while a change that touches no
dependence structure at all moves the ES several times further.}\label{tab:es-decomposition}
\small
\setlength{\tabcolsep}{2mm}
\begin{tabular}{lccc}
\toprule
& \multicolumn{2}{c}{$\Delta$ Energy Score} & \multicolumn{1}{c}{$\Delta$ njNLL} \\
\cmidrule(lr){2-3} \cmidrule(lr){4-4}
\textbf{Dataset} & copula & marginals & copula \\
\midrule
\texttt{USHCN}        & $0.001$ & $0.011$ & $0.166$ \\
\texttt{PhysioNet'12} & $0.013$ & $0.021$ & $0.303$ \\
\texttt{MIMIC-III}    & $0.009$ & $0.072$ & $0.330$ \\
\texttt{MIMIC-IV}     & $0.000$ & $0.037$ & $0.382$ \\
\bottomrule
\end{tabular}
\end{table}

\Cref{tab:es-decomposition} makes the effect visible in our own results, by decomposing the numbers of
\Cref{tab:energy-score} into the two changes that produce them. Attaching \model{} to \margmodel{} changes
the dependence structure and nothing else, since marginalization consistency leaves the marginals where
\margmodel{} put them (\Cref{app:mc-proof}), which \Cref{tab:wd} confirms empirically. Switching
\margmodel{} to the CRPS objective of \Cref{app:crps-training} changes the marginals and leaves the
copula stage untouched. The first change is worth between $0.17$ and $0.38$ njNLL, a gap that separates
\model{} from every jointly trained baseline in \Cref{tab:njnll}. In ES it is worth between $0.000$ and
$0.013$, at or below the fold-to-fold standard deviations reported in \Cref{tab:energy-score}. The second
change improves no dependence structure at all and is worth up to $0.072$ of ES (on \texttt{MIMIC-III}\@). This is an 8 times higher increase compared to adding the copula.

On these datasets the ES column therefore ranks marginal fits, and a model's position in it says little
about the quality of its copula.
\model{} with the CRPS-trained marginals of \Cref{app:crps-training} attains the best ES on
\texttt{MIMIC-III} and \texttt{MIMIC-IV}\@. On \texttt{PhysioNet'12}, it comes in second after ProFITi, a model that does not guarantee
marginalization consistency and has a very strong encoder in GraFITi~\citep{Yalavarthi2024.GraFITi}, which we cannot use as it does not separate targets.

\clearpage
\section{\texorpdfstring{Training \margmodel{} on CRPS}{Training MargFlow on CRPS}}\label{app:crps-training}

The log-likelihood is one proper scoring rule among many, and which rule a model is fitted with decides
what the fit is good at (the \emph{optimum score estimation} view of \citet{Gneiting2007.Strictly}).
Our own numbers show the gap. \margmodel{} attains the lowest mNLL on all four datasets
(\Cref{tab:mnll}), yet on CRPS it is only on par with CircuITS (\Cref{tab:crps}), which it beats
comfortably in likelihood. The marginals of any copula process (incl. \model{}) are a separate module with their own objective, so
we can act on that. Fitting \margmodel{} with CRPS directly leaves the rest of the construction
untouched. The rows marked $^*$ in \Cref{tab:mse,tab:energy-score,tab:crps} are produced this way.

\paragraph{CRPS is available in closed form.}
We train on the integral form of \Cref{eq:crps-def}, the quadratic distance between the predictive CDF
and the step function at the observation, taken at $\hat{F}_n = \cdfmargn{n}$. The DSF
head of \margmodel{} \emph{is} that conditional CDF (\Cref{app:margflow}), and a single forward pass
returns $\cdfmargn{n}(z)$ for any $z$. Evaluating it therefore requires neither sampling from the model
nor inverting the flow, and it is differentiable in the flow parameters $\theta_n$ by ordinary
backpropagation through that forward pass. That is cheap enough to train on. 

\paragraph{Computing the integral.}
We approximate the integral numerically. We evaluate the CDF on a grid of points around the
observation and apply the trapezoid rule. One detail matters. The integrand is ${\cdfmargn{n}(z)}^2$ to
the left of the observation and ${\bigl(1 - \cdfmargn{n}(z)\bigr)}^2$ to the right of it, so it
\emph{jumps} at $z = \ytgt_n$. A grid that runs across this jump is markedly less accurate. We
therefore integrate the two sides separately, each up to a distance $w$ from the observation,
\begin{equation}
    \widetilde{\mathrm{CRPS}}(\cdfmargn{n}, \ytgt_n)
    \;=\;
    \int_{0}^{w} \cdfmargn{n}{(\ytgt_n - s)}^2 \, \mathrm{d}s
    \;+\;
    \int_{0}^{w} {\bigl(1 - \cdfmargn{n}(\ytgt_n + s)\bigr)}^2 \, \mathrm{d}s,
    \label{eq:crps-split}
\end{equation}
and compute each on $J + 1$ equally spaced points.
Since the targets are normalized to zero mean and unit variance (\Cref{app:exp-details}), we take $w = 10$ (ten standard deviations) and $J = 200$, i.e.\ 401
CDF evaluations per target point.

\paragraph{Why a jointly trained model has no analogue of this.}
The freedom exercised above belongs to the decoupled setup and does not transfer. A jointly trained
model has a single objective covering marginals and dependence at once. CRPS cannot take its place,
because it is a univariate score, so substituting it would delete every term that concerns the dependence
structure and the model would stop being a joint model at all. Nor can the two simply be added, since a
log-density and a CRPS differ in units and scale and combining them introduces a weight that has no
principled value. The multivariate analogue, the Energy Score, is not available in closed form for the
models at hand, and its sample estimator~\Cref{eq:es} requires differentiable sampling from the joint
predictive distribution. 

\clearpage
\section{Ablations}\label{app:ablations}

We describe two ablations of \model{}. The first, Mix-GC, replaces the Gaussian Mixture Copula by a mixture of Gaussian copulas and is reported alongside the baselines in the main tables. The second lifts the one restriction that marginalization consistency imposes on \model{}, namely that the mixture weights must not see the target set.

\subsection{The Mix-GC Ablation}\label{app:mix-gc}

Recall from \Cref{sec:gmc} that a GM-C is induced by a single multimodal base density $\pdfgmm(\mathbf{z}) = \sum_{j=1}^{K} \pi_j\, \mathcal{N}(\mathbf{z};\mu_j,\Sigma_j)$ on $\R^N$ via~\Cref{eq:copula} using the GMM's own marginal CDFs. A Mix-GC is instead a convex combination of $K$ independent Gaussian copulas,
\begin{equation}\label{eq:mix-gc}
    \cop(\mathbf{u}) \;=\; \sum_{j=1}^{K} \pi_j\, \cop_{R_j}(\mathbf{u}),
\end{equation}
where $\cop_{R_j}$ is the Gaussian copula induced by a correlation matrix $R_j$. \Cref{app:gmc-expressivity} discusses what this ablation gives up in expressivity and illustrates it on two toy densities. Each Gaussian copula component is centered at the origin in latent space and uses standard-normal marginals. The per-component means $\mu_j$ and standard deviations $\sigma_{j,n}$ of the GM-C construction are therefore no longer present. Only the off-diagonal correlation structure remains.

\paragraph{Architectural changes.}
The \model{}-Mix-GC ablation reuses every component of \model{} described in \Cref{sec:model} unchanged, except for the parts that produce the per-component means and standard deviations. Concretely:
\begin{itemize}
    \item The encoder, the global summary $\bar{\mathbf{h}}$, and the per-target embeddings $\mathbf{e}_n$ are identical to \model{}.
    \item The mixture-weight head $\pi = \mathrm{softmax}(\mathrm{MLP}_\pi(\bar{\mathbf{h}}))$ is identical to \model{}.
    \item The mean function $\phi_\mu$ and the standard-deviation function $\phi_\sigma$ of~\Cref{eq:mu-sigma} are \emph{removed}. Each Gaussian copula component is centered at zero with unit variance by definition, so $\mu_{j,n}$ and $\sigma_{j,n}$ are not free parameters.
    \item The correlation network that produces the per-component feature matrices $V_j \in \R^{N \times H}$ is unchanged. We normalize the feature vectors $\mathbf{v}_{j,n}$ as in \Cref{sec:model} and form the correlation matrices $R_j = \tilde V_j \tilde V_j^\top + \mathrm{diag}{\bigl(1 - {\lVert \tilde{\mathbf{v}}_{j,n} \rVert}^2\bigr)}_{n=1}^{N}$ as in \model{}. The construction then stops, since $R_j$ now carries the full per-component covariance and no scale parameters multiply it.
\end{itemize}
The hyperparameter search space therefore matches that of \model{} (\Cref{tab:hp-gmcp}) with the \texttt{corr\_net\_hidden\_dim} entry reinterpreted as the hidden dimension of the only remaining MLP\@.

\subsection{Target-Dependent Mixture Weights}\label{app:target-weights}

\model{} imposes one restriction in order to be marginalization consistent. The mixture weights $\pi$ are inferred from the context alone, $\pi = \phi_\pi(\Dctx)$, and never see the target set $\Xtgt$ (\Cref{sec:model}, Step~3 of the proof of \Cref{thm:mc}). Every other parameter function is free to read its own target input. We measure what the restriction costs by lifting it.

\paragraph{The ablation.}
Recall from~\Cref{eq:enc-pool} that \model{} obtains the global summary $\bar{\mathbf{h}}$ by pooling the per-channel states $\tilde{\mathbf{H}}$, which are a function of $\Dctx$ alone, with a single learnable attention query $\mathbf{q}$. The ablation uses mean-pooling on the \emph{target encodings} instead,
\begin{equation}
    \bar{\mathbf{e}} \;=\; \frac{1}{N}\sum_{n=1}^{N} \e_n
    \;=\; \frac{1}{N}\sum_{n=1}^{N} \enc\bigl(\xtgt_n, \Dctx\bigr) \;\in\; \R^{D},
    \qquad
    \pi \;=\; \mathrm{softmax}\bigl(\mathrm{MLP}_\pi(\bar{\mathbf{e}})\bigr) \;\in\; \Delta^{K}.
    \label{eq:abl-pool}
\end{equation}
Nothing else changes. The encoder is identical up to and including the target encodings of~\Cref{eq:enc-target}, the mixture-weight head $\mathrm{MLP}_\pi$ keeps its architecture, and $\phi_\mu$, $\phi_\sigma$ and $\phi_V$ are untouched, so $\mu_j$, $\sigma_j$ and $\Sigma_j$ are formed as in \Cref{sec:model}.

\begin{table}[ht]
\centering
\caption{Test \textbf{njNLL} (lower is better) of \model{}, whose mixture weights see the context only, against the target-dependent-weight ablation of~\Cref{eq:abl-pool}, which is marginalization inconsistent. Mean and standard deviation over the five splits; both arms share the same frozen \margmodel{} marginals. Neither row is marked best.}\label{tab:target-weights}
\small
\setlength{\tabcolsep}{1.5mm}
\begin{tabular}{lrrrr}
\toprule
\textbf{Mixture weights} & \multicolumn{1}{c}{\texttt{USHCN}} & \multicolumn{1}{c}{\texttt{PhysioNet'12}} & \multicolumn{1}{c}{\texttt{MIMIC-III}} & \multicolumn{1}{c}{\texttt{MIMIC-IV}} \\
\midrule
\model{}: $\pi = \phi_\pi(\Dctx)$        & \metric{-4.135}{0.304} & \metric{-0.745}{0.019} & \metric{-0.835}{0.218} & \metric{-2.201}{0.058} \\
Ablation: $\pi = \phi_\pi(\Dctx, \Xtgt)$ & \metric{-4.143}{0.331} & \metric{-0.740}{0.022} & \metric{-0.846}{0.217} & \metric{-2.207}{0.058} \\
\midrule
Difference (ablation $-$ \model{})       & $-0.008$ & $+0.005$ & $-0.011$ & $-0.006$ \\
\bottomrule
\end{tabular}%
\end{table}

\paragraph{Results.}
\Cref{tab:target-weights} shows that conditioning the mixture weights on the target set moves the njNLL very little. The two approaches differ by at most $0.011$ njNLL, on MIMIC-III\@, and the sign of the difference is not uniform. The ablation is nominally ahead on USHCN, MIMIC-III and MIMIC-IV, and behind on PhysioNet'12.

\clearpage
\section{Sensitivity Analysis}\label{app:sensitivity}

We study the sensitivity of \model{} to the two hyperparameters that are specific to the GM-C copula module, namely the number of mixture components $K$ and the per-component correlation-network hidden dimension $H$ (i.e.\ the width of the MLPs that emit $\mu_j$, $\sigma_j$, and the rows of the Gram-matrix factors $V_j \in \R^{N \times H}$, see \Cref{tab:hp-gmcp}). For each sweep, we vary the analyzed hyperparameter while fixing all remaining hyperparameters to the values selected by the random search described in \Cref{app:hyperparams}. Each configuration is trained from scratch on all five splits, and we report the mean test njNLL across the five splits.

\paragraph{Number of mixture components $K$.} \Cref{fig:sensitivity-k} sweeps $K \in \{1, 3, 5, 7, 10\}$ on all four datasets. The case $K = 1$ collapses the GM-C to a Gaussian copula and is inferior on every dataset. Moving from $K = 1$ to $K = 3$ produces a large drop in njNLL on all four benchmarks, and the curve flattens beyond $K \ge 3$. 
\begin{figure}[h]
    \centering
    \includegraphics[width=\linewidth]{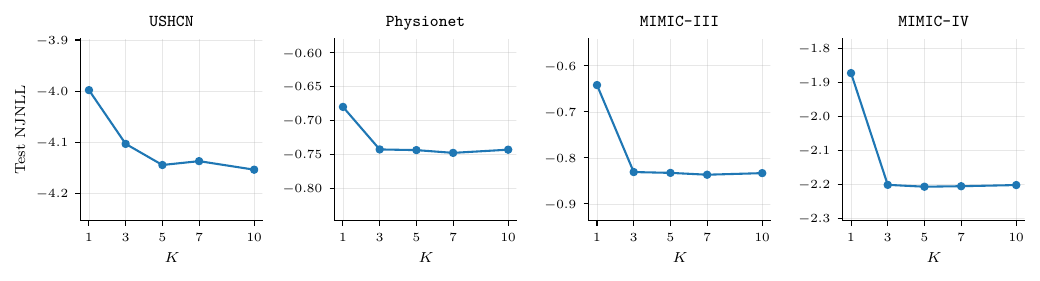}
    \caption{Sensitivity of \model{} to the number of mixture components $K$. Test njNLL (lower is better) averaged over the five splits, with all remaining hyperparameters fixed to the values selected by the random search of \Cref{app:hyperparams}.}\label{fig:sensitivity-k}
\end{figure}

\paragraph{Correlation-network hidden dimension $H$.} \Cref{fig:sensitivity-h} sweeps $H \in \{16, 32, 64, 128\}$. Across all four datasets the test njNLL is flat in $H$. The only visible deviation is a small uptick on USHCN at $H = 128$, which we attribute to mild overfitting on the smallest dataset.

$H = 16$ suffices on MIMIC-III and MIMIC-IV even though many test instances have $N > H$ in those datasets (the maximum target count $N$ is $85$ for MIMIC-III and $79$ for MIMIC-IV\@; see \Cref{tab:dataset-summary}). In this regime, each per-component covariance $\Sigma_j$ is built from the rank-$H$ Gram matrix of $\tilde V_j \in \R^{N \times H}$ in~\Cref{eq:sigma-j}, so its off-diagonal block has rank at most $H$. That is the limitation we flag in the main text, that \model{} cannot represent full-rank or near-perfect dependence when $H \ll N$\@. The flat sensitivity curves show that the rank ceiling does not bind in practice. It may even help, since not all variables depend on each other.

\begin{figure}[h]
    \centering
    \includegraphics[width=\linewidth]{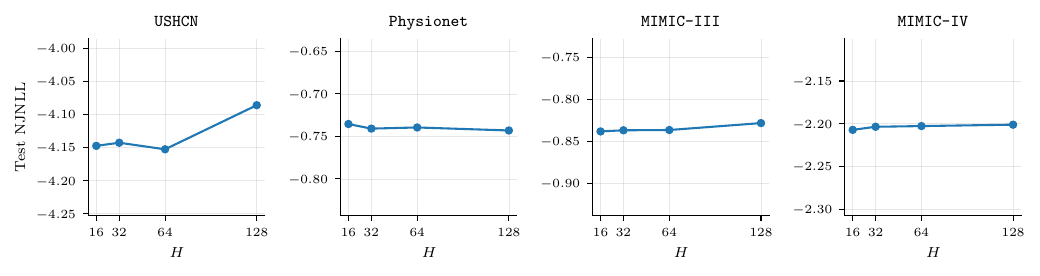}
    \caption{Sensitivity of \model{} to the correlation-network hidden dimension $H$.
    }\label{fig:sensitivity-h}
\end{figure}

\clearpage
\section{Learning Curves}\label{app:learning-curves}

\Cref{fig:learning-curves} gives the validation mNLL learning curves of \model{}-Joint and of the decoupled \margmodel{} + \model{} setup on MIMIC-III, as a complement to the marginal-quality analysis in \Cref{sec:experiments}. \model{}-Joint shares the architecture of \model{} and trains all weights jointly under njNLL. The decoupled setup fits \margmodel{} in isolation first, then trains the copula on top of the frozen marginals. All models were trained until the validation njNLL stopped improving for 30 epochs.

\begin{figure}[h]
    \centering
    \includegraphics[width=\textwidth]{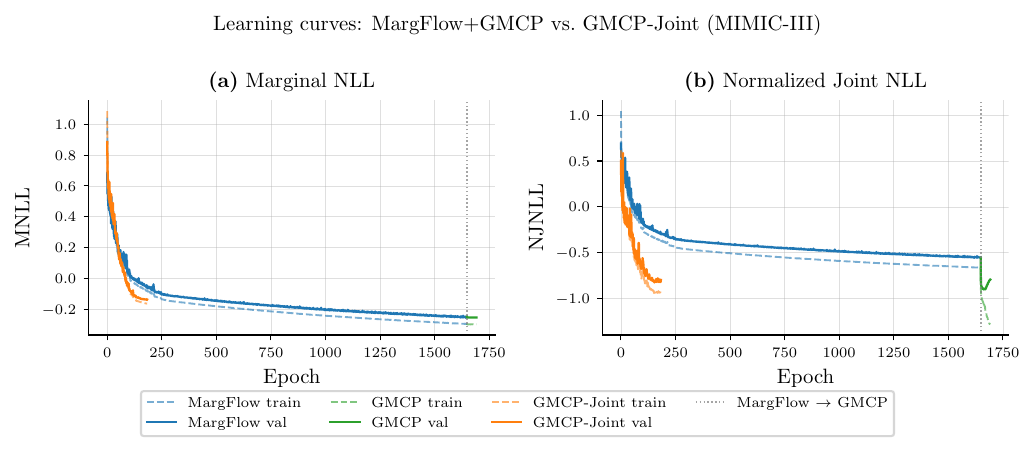}
    \caption{Validation learning curves of \model{}-Joint and the decoupled \margmodel{} + \model{} setup on MIMIC-III\@. Training the marginals in isolation reaches a lower marginal NLL than training all weights jointly under njNLL.}\label{fig:learning-curves}
\end{figure}

\end{document}